\documentclass{article}

\usepackage{amsmath}
\usepackage{graphicx, tipa} 
\usepackage{algorithm, comment}
\usepackage{algpseudocode, enumerate}
\usepackage{amsfonts, amsthm, bbm}
\usepackage{thmtools}
\usepackage{thm-restate}
\usepackage{hyperref}
\usepackage{cleveref}

\newtheorem{definition}{Definition}

\newtheorem{theorem}{Theorem}

\newtheorem*{remark}{Remark}

\newtheorem{assumption}{Assumption}
\newtheorem{proposition}{Proposition}
\usepackage{dsfont}
\usepackage{array}
\usepackage{hhline}
\usepackage{multirow}

\newcommand{\supp}{\textup{supp}}
\newcommand{\vol}{\textup{vol}}
\newcommand{\diam}{\textup{diam}}
\newcommand{\cl}{\text{cl}}
\newcommand{\vc}{\textup{VC}}
\newcommand{\dist}{\textup{dist}}
\newcommand{\inter}{\textup{int}}
\newcommand{\prob}{\mathbb{P}}
\newcommand{\reals}{\mathbb{R}}

\newcommand{\nats}{\mathbb{N}}
\newcommand{\expect}{\mathbb{E}}
\newcommand{\cmmnt}[1]{\ignorespaces}

\DeclareMathOperator*{\argmax}{arg\,max}

\usepackage[preprint, nonatbib]{neurips_2024}

\usepackage[utf8]{inputenc} 
\usepackage[T1]{fontenc}    
\usepackage{hyperref}       
\usepackage{url}            
\usepackage{booktabs}       
\usepackage{amsfonts}       
\usepackage{nicefrac}       
\usepackage{microtype}      
\usepackage{xcolor}         

\title{A Closer Look at the Learnability of Out-of-Distribution (OOD) Detection}

\author{%
  Konstantin Garov \\
  University of California - San Diego\\
  \texttt{kgarov@ucsd.edu} \\
  \And
Kamalika Chaudhuri \\
  University of California - San Diego\\
  \texttt{kamalika@cs.ucsd.edu} \\
}

\begin{document}

\maketitle

\begin{abstract}
  Machine learning algorithms often encounter different or ``out-of-distribution'' (OOD) data at deployment time, and OOD detection is frequently employed to detect these examples. While it works reasonably well in practice, existing theoretical results on OOD detection are highly pessimistic. In this work, we take a closer look at this problem, and make a distinction between uniform and non-uniform learnability, following PAC learning theory. We characterize under what conditions OOD detection is uniformly and non-uniformly learnable, and we show that in several cases, non-uniform learnability turns a number of negative results into positive. In all cases where OOD detection is learnable, we provide concrete learning algorithms and a sample-complexity analysis.%
\end{abstract}

\section{Introduction}

The statistical learning framework, which lies at the core of much of machine learning, states that there is an underlying data distribution, from which both training and test data are drawn; the goal of the learner is then to achieve high performance on this distribution. This statistical learning assumption, however, rarely holds in practice. New kinds of inputs that were absent from the training data often materialize at deployment; for example, a self-driving car may encounter novel driving conditions \cite{filos2020can, sun2020scalability, bojarski2016end}; or, a medical-image classifier may encounter distributions shifts such as: different demographic of patients \cite{finlayson2021clinician, chen2021ethical}, difference in hospital protocols \cite{subbaswamy2021evaluating}, and others \cite{nestor2019feature, zech2018variable}. These kinds of different ``unseen'' examples are known in the literature as out-of-distribution or OOD examples.

The simplest way to address the challenge of OOD examples is to detect them, and then abstain from classification \cite{chow1970optimum}. This has led to a large body of empirical work on OOD detection that works reasonably well in practice \cite{ liu2020energy, ren2019likelihood, devries2018learning, fort2021exploring, sun2022out, vyas2018out, winkens2020contrastive, sun2021react, lin2021mood, Graham_2023_CVPR, zhang2023decoupling, liu2023good}. Theoretically, however, OOD detection is much less well-understood. \cite{fang2022out} formalized OOD detection as the problem of distinguishing between two distributions $D_{in}$ and $D_{out}$. $D_{in}$ is the {\em{ID distribution}}, and is observed during training, whereas $D_{out}$, the {\em{OOD distribution}}, is a novel distribution that only appears at test time with some frequency. The goal of the learner is to learn a rule to distinguish examples from $D_{in}$ and $D_{out}$ at test time. 

Unfortunately, most of the known theoretical results on OOD detection are highly pessimistic. \cite{fang2022out} showed that even under fairly stringent assumptions, OOD detection may still be impossible; moreover, while they do provide some positive results, the majority of them are purely learnability results \cmmnt{with no associated learning algorithms}. This contrasts with what has been observed in practice, where OOD detection methods have shown to achieve reasonable performance on a number of different benchmarks \cite{yang2021generalized}. In this paper, we attempt to bridge this gap between theoretical understanding and practical performance of OOD detection. To this end, we take a closer look at the problem with the view towards characterizing when OOD detection is possible in a structured and principled manner.

We begin by going back to the definition of learnability for OOD detection, and in line with classical theory of supervised learning, we make a distinction between two kinds of learnability -- uniform, where there is a uniform bound on the number of samples for a fixed expected risk, and non-uniform, where this is not the case. This allows a more nuanced understanding as we find that in some cases OOD detection is only non-uniformly learnable. 

Clearly, perfect learning OOD detection with no risk is impossible in the most general setting when the supports of the in-distribution and out-distribution overlap \cite{fang2022out}. We show that it may not even be non-uniformly learnable when these two supports are disjoint. To better understand learnability, we then investigate a series of further natural \cmmnt{and relatively mild} assumptions on the two distributions. We characterize under what conditions OOD detection is uniformly and non-uniformly learnable, and we find that in several cases, non-uniform learnability turns a number of negative results into positive. In all cases where OOD detection is learnable, we provide concrete learning algorithms. 

Specifically, our contributions are as follows:
\begin{itemize}
    \item We introduce uniform and non-uniform learnability and extend previous results to show that OOD detection is not always learnable even in a non-uniform sense when the supports of the ID and OOD distributions are disjoint. Doing so, we disprove a conjecture of ~\cite{fang2022out} . 

    \item We provide a non-uniform OOD detector for Far-OOD detection, where the supports of the ID and OOD distributions are separated by a distance more than a constant $\tau > 0$. We also provide a uniform learner when the diameters of all supports of ID distributions are less than a constant $R > 0$.
    
    \item We provide non-uniform OOD detectors when the ID and OOD supports are disjoint and, additionally, either all ID or all OOD distributions have a Hölder continuous density function. As above, when the diameters of the supports of all ID distributions are uniformly bounded, we provide uniform learners. 
    
    \item We provide a non-uniform learner of OOD detection if all ID distributions have a convex support disjoint from the supports of their respective OOD distributions.
    
    \item We extend the popular No Free Lunch Theorem in learning theory to the context of OOD detection.
\end{itemize}

Our results suggest that OOD detection may not be as impossible in theory as previously thought; it is plausible that the practical methods may be succeeding because in reality, the data distributions and the OOD detection benchmarks have certain good properties, and also because they might be working with ``average-case'' samples. Further investigation into the nature of these algorithms is left for future work. 

\section{Preliminaries}
\label{sec:prelims}

We begin by introducing some basic concepts and notation. Similar to the standard PAC learning setup, we call the set $\mathcal{X}$ of all possible inputs the \textbf{instance space}. 

Let $D_{in}$ be a distribution over $\mathcal{X}$ that we call the \textbf{ID (in-distribution) distribution}. At training time, our learning algorithm $A$ has access to an example oracle $EX(D_{in})$ from which it can draw independent samples $x \in \mathcal{X} \sim D_{in}$. In addition, there is an \textbf{OOD (out-of-distribution) distribution} $D_{out}$ over $\mathcal{X}$ which we do not observe during training. 

At test time, we get samples from $D_{\alpha}$ - a mixture of $D_{in}$ and $D_{out}$: $D_{\alpha} =  (1 - \alpha) D_{in} + \alpha D_{out}$ for some $\alpha \in (0,1)$, which we will call the \textbf{OOD probability}. At the same time a pseudolabel $y(x) \in \{0,1\}$ is recorded indicating whether $x$ was generated from $D_{in}$ ($y(x) = 1$) or $D_{out}$ ($y(x) = 0$). Our goal is to design an algorithm to predict these pseudolabels at test time. In practice, an instance $x$ may also have a categorical label as in the standard supervised setting; for our purposes, however, we ignore this label. In \nameref{sec:discdef} we discuss this scenario in more detail.

We remark that this is broadly equivalent to a PAC learning setup in which data from one of the labels is designated as out-of-distribution and is filtered out and not observed during training.


\subsection{Some Necessary Definitions}




Following \cite{fang2022out}, we call a \textbf{domain} a pair of distributions $(D_{in}, D_{out})$ over an instance space $\mathcal{X}$. A \textbf{domain space} $\mathcal{D}$ over $\mathcal{X}$ is a collection of such domains. The learning algorithm has knowledge of the domain space $\mathcal{D}$ but not of the specific domain used to generate the training and test data. 


A \textbf{hypothesis} is a classifier $h: \mathcal{X} \rightarrow \{0, 1\}$ that maps an instance to a pseudolabel. A \textbf{hypothesis class} is a collection of hypotheses. Abusing the notation a bit, we also interpret a subset $h \subseteq \mathcal{X}$ as the hypothesis specified by the indicator function of $h$.

We use the 0-1 loss function $l: \{0,1\}^2 \rightarrow \{0,1\}$, defined as $l(y_1, y_2) = \mathds{1}\{y_1 = y_2\}$, to evaluate the performance of a hypothesis. For any $h: \mathcal{X} \rightarrow \{0,1\}$ we then define its \textbf{risk} $R^{\alpha}_D(h)$ with respect to a domain $D = (D_{in}, D_{out})$ and OOD probability $\alpha \in (0,1)$ to be:

\begin{equation*}
    R^{\alpha}_D(h) = \expect_{x \sim (1-\alpha)D_{in} + \alpha D_{out}}[l(h(x), y(x))]
\end{equation*}

We also define the in-distribution and and out-of-distribution risks to be:

\begin{equation*}
    R^{in}_D(h) = \expect_{x \sim D_{in}}[l(h(x), 1)] \quad \mathrm{and} \quad R^{out}_D(h) = \expect_{x \sim D_{out}}[l(h(x), 0)]
\end{equation*}

We observe that $R^{\alpha}_D(h) = (1-\alpha) R^{in}_D(h) + \alpha R^{out}_D(h)$.


Extending of the concept of PAC learnability for supervised learning, the goal of OOD detection is to find a hypothesis $h$ in a hypothesis class $\mathcal{H}$ that achieves approximately the lowest possible risk among all elements of $\mathcal{H}$. We formalize this with the following definition:

\begin{definition}[Uniform learnability of OOD detection] 
\label{def:unif}
Given a domain space $\mathcal{D}$, a hypothesis space $\mathcal{H}$, and an OOD probability $\alpha \in (0,1)$ we say that OOD detection is \textbf{uniformly learnable} for $\mathcal{D}$ in $\mathcal{H}$ if there exists a function $N: (0, 1/2)^2 \rightarrow \nats$ and a learning algorithm $\mathcal{A}$ such that for every $0< \epsilon < 1/2$ and $0< \delta < 1/2$, every domain $D = (D_{in}, D_{out}) \in \mathcal{D}$, if $\mathcal{A}$ is given access to EX($D_{in}$) and inputs $\epsilon$ and $\delta$, $\mathcal{A}$ makes at most $N(\epsilon, \delta)$ calls to the oracle EX($D_{in}$) and outputs a hypothesis $h \in \mathcal{H}$, which with probability of at least $1 - \delta$ satisfies:

\begin{equation*}
    R^{\alpha}_D(h) \leq \inf_{h' \in H} R^{\alpha}_D(h') + \epsilon
\end{equation*}

where the probability is taken over the random samples drawn from $EX(D_{in})$ and the possible internal randomization of $\mathcal{A}$. 

\end{definition}

We say that OOD detection is \textbf{realizable} for a domain space $\mathcal{D}$ in a hypothesis space $\mathcal{H}$ if $\inf_{h' \in H} R^{\alpha}_D(h') = 0$ holds for all domains $D \in \mathcal{D}$  .
 
In addition, we introduce the concept of non-uniform learnability of OOD detection again extending the supervised learning theory \cite{shalev2014understanding} in which the sample complexity of the learner for fixed values of $\epsilon$ and $\delta$ is not uniformly bounded. We later observe that this relaxation leads to big differences in the learnability of OOD detection:

\begin{definition}[Non-uniform learnability of OOD detection] 
\label{def:nonunif}
Given a domain space $\mathcal{D}$. a hypothesis space $\mathcal{H}$, and an OOD probability $\alpha \in (0,1)$  we say that OOD detection is \textbf{non-uniformly learnable} for $\mathcal{D}$ in $\mathcal{H}$ if there exists a learning algorithm $\mathcal{A}$ such that for every $0 < \epsilon< 1/2$ and $0 < \delta< 1/2$
, if $\mathcal{A}$ is given access to EX($D_{in}$), $\epsilon$, and $\delta$, $\mathcal{A}$ terminates with probability $1$ and outputs a hypothesis $h \in \mathcal{H}$, which with probability of at least $1 - \delta$ satisfies:

$$R^{\alpha}_D(h) \leq \inf_{h' \in H} R^{\alpha}_D(h') + \epsilon$$

where the probability is taken over the random samples drawn from $EX(D_{in})$ and the possible internal randomization of $\mathcal{A}$.

\end{definition}

A feature of Definitions \ref{def:unif} and \ref{def:nonunif} is that the OOD probability $\alpha$ is known a priori; since this is often not the case in practice, a natural question is whether it makes a difference in learnability. In Theorem \ref{thmodes} in \nameref{sec:discdef}, we show that in a realizable setting OOD detection is learnable with a priori known OOD probability iff it is learnable with a priori unknown OOD probability. We remark, however, that this is not the case in an agnostic setting, and, thus, our definitions further relax the notion of learnability (compare to \cite{fang2022out}). For example, OOD detection can be learnable even for domains with overlapping supports.

Lastly, we remark that in a realizable setting, the exact hypothesis space $\mathcal{H}$ is also of no significance. That is, as long as the domain space $\mathcal{D}$ remains realizable in $\mathcal{H}$ the the learnability of OOD detection does not depend on the exact choice for $\mathcal{H}$. Therefore, as in this work we focus on the learnability of realizable OOD detection, we assume that $\mathcal{H} = \mathcal{P}(\mathcal{X})$, the set of all subsets of $\mathcal{X}$, and we do not discuss the exact value of the OOD probability. Further details alongside proofs of all the claims made above are provided in \nameref{sec:discdef}.

\section{A Summary of the Learnability Results}
\label{sec:summary}

Before we get to our main results, we highlight some conditions under which we will investigate the learnability of OOD detection as it is clear that without any restriction on the domain and the hypothesis spaces, OOD detection is not possible. Intuitively, the set of domains consistent with samples $S$ can be too large and diverse and there might be no single hypothesis which achieves low risk with respect to all such domains.

\subsection{Assumptions}
\label{sec:assumptions}

In particular, we introduce several assumptions on domains. We also extend them to domain spaces $\mathcal{D}$: we say that a domain space satisfies a particular assumption if all domains in $\mathcal{D}$ satisfy it.

\paragraph{Separation Assumptions.} We first look at assumptions that relate to how the ID and OOD distributions are separated in space. A plausible assumption often considered in the context of OOD detection is the following:

\begin{assumption}\label{ass:dsa}We say that a domain $D$ over a topological space satisfies the \textbf{Disjoint Supports Assumption (DSA)} if it has disjoint ID and OOD supports, that is, if:
    $\supp(D_{in}) \cap \supp(D_{out}) = \emptyset$

\end{assumption}

We note that a domain spaces $\mathcal{D}$ over $\mathcal{X}$ which satisfies DSA is always realizable in $\mathcal{P}(\mathcal{X})$. DSA is further strengthen by:

\begin{assumption}
We say that a domain $D$ over a metric space satisfies the \textbf{$\tau$ - Far Supports Assumption ($\tau$-FSA)} if:
  $\dist(\supp(D_{in}), \supp(D_{out})) > \tau $

\end{assumption}

Generally, we observe such separation: when the ID and OOD data have some inherent differences such as different semantic labels, different styles, or other; or in the representation space of a suitably trained neural network \cite{winkens2020contrastive}. In particular, many current practical algorithms (such as all Distance Methods discussed in the \nameref{relwork} section) rely on some kind of geometric separation between the ID and OOD data. This is also true for many popular OOD detection benchmarks, such as Imagenet \cite{deng2009imagenet} as ID, and Imagenet-O \cite{hendrycks2021natural} as OOD.

\paragraph{Structural Assumptions.} A second class of plausible in the context of OOD detection assumptions are geometric restrictions on the support of the ID distribution. In this work, we consider two such assumptions:

\begin{assumption}
We say that domain $D$ over $\reals^n$ satisfies the \textbf{Convex ID  Support Assumption (ConvexID)} if the support of $D_{in}$ is a convex subset of $\reals^{n}$.
    
\end{assumption}

\begin{figure}[t!]
    \includegraphics[width=151mm,scale=1.5]{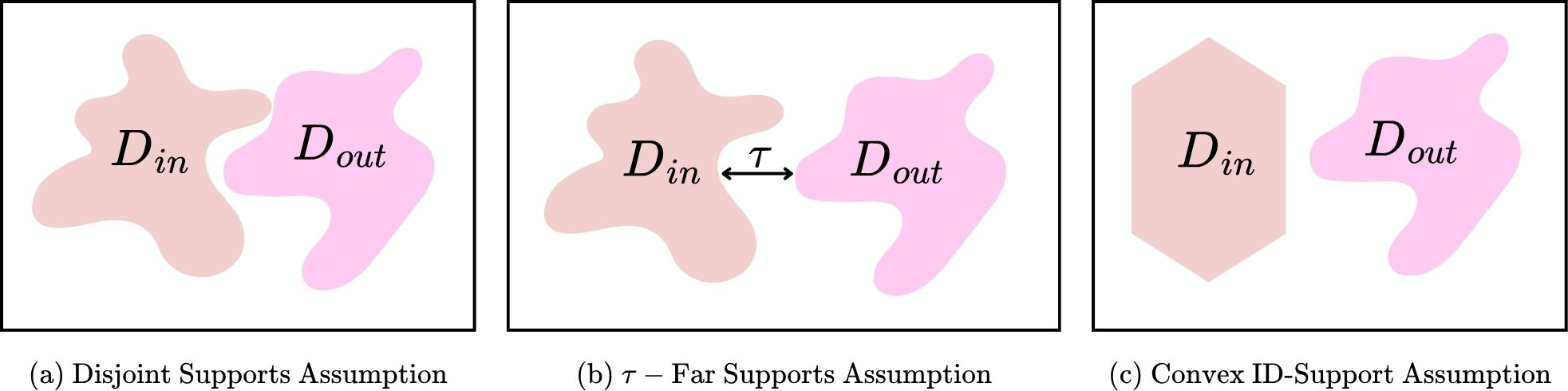}
    \caption{Examples of domain restrictions}
    \centering
\end{figure}

\begin{assumption}
For a positive real constant $R$ we say that a domain $D$ over $\reals^n$ satisfies the \textbf{Bounded ID  Support Assumption (BoundedID)}  for $R$ if:
    $\diam(\supp(D_{in})) < R$
\end{assumption}

\paragraph{Distributional Assumptions.} Intuitively, the complexity of a domain $D$ can be decomposed into two factors: the supports of $D_{in}$ and $D_{out}$ and the actual distributions within these supports. Having already introduced restrictions of the former above, we now look at the following distributional assumptions:  

\begin{assumption}
We say that a domain $D$ over $\reals^n$ satisfies the \textbf{Absolutely Continuous ID  distribution Assumption (ContID)} if the ID distribution $D_{in}$ is absolutely continuous with respect to the Lebesgue measure on $\reals^{n}$. By Radon-Nikodym Theorem \cite{cohn2013measure} this is equivalent to $D_{in} = f(x)dx$ for some density function $f$ on $\reals^n$.
\end{assumption}

\begin{assumption}
For real $\alpha>0$ and $C \geq 0$ we say that a domain $D$ over $\reals^n$ satisfies the \textbf{$(\alpha, C)$-Hölder-continuous ID  Distribution Assumption (ID  Hölder)} if $D_{in}$ has a $(\alpha, C)$-Hölder continuous density function.

Similarly, we say that $D$ satisfies the \textbf{$(\alpha, C)$-Hölder-continuous OOD Distribution Assumption (OOD Hölder)} if $D_{out}$ has a $(\alpha, C)$-Hölder continuous density function.

\end{assumption}

We remark that most standard distribution modeling naturally generated data (such as the Gaussian, Beta, ant Exponential Distributions) satisfy both of these assumptions for a suitable choice of the parameters.

\subsection{A Principled Learner}
\label{sec:procedures}

A fundamental question in OOD detection is: \textit{What can a principled algorithm for learning OOD detection look like?} To answer this, we introduce the following natural procedure, that is a map of the form $\mathcal{B}: \mathcal{X}^* \rightarrow \mathcal{H}$:

\begin{definition} We call \textbf{Maximal Zero OOD Risk Procedure} the map $\mathcal{B}: \mathcal{X}^* \rightarrow \mathcal{P}(\mathcal{X})$ defined by:

\begin{equation*}
    \mathcal{B} (S) = \bigcap_{D \in \mathcal{D}, S \subseteq \supp(D_{in})} \supp(D_{out})^c
\end{equation*}

That is, given samples $S \in \mathcal{X}^*$, the procedure $\mathcal{B}$ looks at all domains that could have generated $S$, and returns the intersection of the complement of the supports of the OOD distributions of these domains. 

\end{definition}

We note that for any domain space the output of the Maximal Zero OOD Risk Procedure always has OOD risk equal to 0.  A natural subsequent question is how to convert a procedure into a learning algorithm. For some domain spaces it might be the case that for a sufficiently large sample set $S$ the hypothesis $B(S)$ has small risk with high probability. A truncated version of the procedure $B$ will then directly yield a uniform learner. 

On the other hand, even if such guarantees are not possible, it might still be possible to construct a non-uniform learner. We formalize this in the following proposition:

\begin{restatable}{proposition}{prtol}
    \label{prtol}
     Let $\mathcal{D}$ be a domain space over $\mathcal{X}$ and let $\mathcal{B}: \mathcal{X}^* \rightarrow \mathcal{P}(\mathcal{X})$ be a procedure such that for all domains $D \in \mathcal{D}$ the ID risk of $\mathcal{B}(S)$ converges in probability to $0$ as the number of samples tends to infinity. 
     \cmmnt{That is: for every $D \in \mathcal{D}$ and $\epsilon > 0$:

    \begin{equation*}
        \prob_{S \sim D^{n}_{in}}[R^{in}_{D}(\mathcal{B}(S)) > \epsilon] \rightarrow 0 \text{ as } n \rightarrow \infty
    \end{equation*}}
    Additionally, assume that for every $\epsilon > 0$ there exists a positive integer $N$ such that for all $n > N$:
    
    \begin{equation*}
        \prob_{S \sim D^{n}_{in}}[R^{out}_{D}(\mathcal{B}(S)) > \epsilon] = 0
    \end{equation*}

    holds for all domains $D$ in $\mathcal{D}$. Then OOD detection for $\mathcal{D}$ is non-uniformly learnable.
\end{restatable}

\subsection{Summary of the Results}

With our assumptions and a conceptual algorithm in place, we are now ready to provide a summary of our technical results. Table 1 summarizes the learnability of OOD detection under the assumptions introduced above. 

\renewcommand{\arraystretch}{2}
\begin{table}
    \centering
    \caption{Learnability results}
    \label{sample-table}
      \begin{tabular}{>{\centering\arraybackslash}m{3.2cm}>{\centering\arraybackslash}m{4.6cm}>{\centering\arraybackslash}m{4.4cm}}
    \toprule
    \cmidrule(r){1-2}
             &  \textbf{Non-uniform Learnability}   & \textbf{Uniform Learnability} \\
    \midrule
    \textbf{DSA (\ref{sec:dsano})} &  \multicolumn{2}{c}{Not learnable} \\  
    \textbf{$\tau$-FSA (\ref{sec:farood})} &  Learnable via the Maximal Zero OOD Risk Procedure & Learnable under BoundedID \\
    \textbf{DSA + Hölder (\ref{sec:holder})} &  Learnable via a variation of the Maximal Zero OOD Risk Procedure & Learnable under BoundedID \\
    \textbf{DSA + ConvexID (\ref{sec:convno})} & \multicolumn{2}{c}{Not learnable over $\reals^n$ for $n \geq 2$} \\
     \textbf{DSA + ConvexID + ContID (\ref{sec:convcont})} &  Learnable via the Maximal Zero OOD Risk Procedure & Not learnable even under BoundedID \\
    \bottomrule
  \end{tabular}
\end{table}

We remark that as similar OOD detection learnability results hold under DSA + ID  Hölder and DSA + OOD Hölder, we represent them in a single row. The exact learning 
algorithms, however,  differ in these two cases.

\section{OOD detection and VC-dimension} \label{sec:nfl}

In supervised learning the connection between PAC learnability and the VC dimension of the concept class is well understood: the Fundamental Theorem of Statistical Learning states that a concept class $\mathcal{C}$ realized in a hypothesis space $\mathcal{H}$ is PAC learnable iff $\mathcal{C}$ has a finite VC dimension \cite{shalev2014understanding}. Thus, a natural question is whether a similar result holds in the context of OOD detection.

As a domain $D$ would generally split its instance space $\mathcal{X}$ into four regions: $\supp(D_{in}) \cap \supp(D_{out})$, $\supp(D_{in}) \setminus \supp(D_{out})$, $\supp(D_{out}) \setminus \supp(D_{in})$, and $\supp(D_{out})^c \cap \supp(D_{in})^c$; (or three under DSA) we need to modify the definition of VC dimension to capture this additional complexity.  

\begin{definition}
\label{def:vc}
We say that a subset $S$ of the instance space $\mathcal{X}$ is \textbf{shattered} by a domain space $\mathcal{D}$ if for every subset $A \subseteq S $ there exists a domain $D \in \mathcal{D}$ such that:

\begin{equation*}
    \supp(D_{in}) \cap S = A \text{ and } \supp(D_{out}) \cap S = S \setminus A
\end{equation*}

The \textbf{VC dimension} of $\mathcal{D}$ is then the cardinality of the largest set $S \subseteq \mathcal{X}$ shattered by $\mathcal{D}$. If $\mathcal{D}$ shatters arbitrarily large sets $S$, we set $\vc(\mathcal{D}) = \infty$.
    
\end{definition}

Unfortunately, the positive direction of the Fundamental Theorem of Statistical Learning does not hold; that is, a finite VC dimension of $\mathcal{D}$ does not necessarily imply even the non-uniform learnability of OOD detection . Indeed, in the proof of Theorem \ref{dsa:no} we show the impossibility of non-uniform learnability of OOD detection for a domain space of VC dimension 1. For a more detailed discussion of this direction refer to \nameref{sec:oodvc}.

On the other hand, a variation of the opposite direction also known as the No Free Lunch Theorem holds. We call a domain space $\mathcal{D}$ \textbf{closed-under-mass-shifting} if for any $D \in \mathcal{D}$ all possible domains $D'$ with $\supp(D_{in}) = \supp(D'_{in})$ and $\supp(D_{out}) = \supp(D'_{out})$ also belong to $\mathcal{D}$. That is, given an ID and an OOD support either all domains with those supports belong to $\mathcal{D}$ or none of them. We then have:

\begin{restatable}[No Free Lunch Theorem for OOD Detection]{theorem}{freelunch}
\label{freelunch}
Let $\mathcal{D}$ be a closed-under-mass-shifting domain space over $\mathcal{X}$ of infinite VC dimension. Then uniform learnabilily of OOD detection is impossible for $\mathcal{D}$ in $\mathcal{P}(\mathcal{X})$.
\end{restatable}

In \nameref{sec:oodvc} we present a generalized version of the above No Free Lunch Theorem for OOD detection.

\section{Results of Table 1}\label{ass_res}

We now discuss the learnability results from Table 1 in more detail.

\subsection{The Disjoint Supports Assumption is not enough for OOD Detection}
\label{sec:dsano}
 
For a fixed $n \in \nats$ we consider the domain space $\mathcal{D}_S$ of all domains $D$ over $\reals^n$ satisfying DSA, also referred to as the separate domain space. \cite{fang2022out} showed that OOD detection $\mathcal{D}_S$ is not uniformly learnable under DSA in hypothesis spaces of finite VC-dimension; we strengthen their result to show that the impossibility result applies to non-uniform learnability and in the absence of any restrictions on the hypothesis space. 

\begin{restatable}{theorem}{dsano}
\label{dsa:no}
The Disjoint Support Assumption does \textbf{not} always imply the non-uniform learnability of OOD detection. In particular, non-uniform OOD detection is impossible for $\mathcal{D}_{S}$ in $\mathcal{P}(\reals^n)$.
\end{restatable}

The underlying reason behind this impossibility result is that the set of ID distributions of the domains in $\mathcal{D}_S$ is dense. Therefore, for any finite set of samples $S$ there are infinitely many domains in $\mathcal{D}_S$ consistent with $S$. As the algorithm cannot distinguish between them it has to be conservative with respect to all of them as otherwise the OOD risk may become very high. But this forces a high ID risk. This differs conceptually from the argument in Theorem 5 from \cite{fang2022out}, which relies on the insufficient expressivity of the hypothesis class assumed to be of finite VC dimension.  

A conjecture of \cite{fang2022out} is that realizable OOD detection is uniformly learnable if $\mathcal{H}$ is agnostically learnable for supervised learning, that if $\mathcal{H}$ is of finite VC dimension. Proving Theorem \ref{dsa:no} we use a domain space, realizable in the hypothesis class $\mathcal{H} = \{[0,1] \setminus x \mid x \in [0,1]\}$ whose VC dimension is $1$. Thus, we disprove the above conjecture. A detailed discussion can be found in \nameref{sec:proofs}, where we also show that \textbf{no} learning rule achieves expected risk less than $\alpha(1 - \alpha)$ for all domains in $\mathcal{D}_{S}$.

\subsection{Far-OOD Detection}
\label{sec:farood}

OOD detection under the $\tau$-FSA is known in the literature as $\tau$-Far OOD detection. In practice, many of the successful methods rely on some kind of geometric separation between $\supp(D_{in})$ and $\supp(D_{out})$  \cite{lee2018simple, podolskiy2021revisiting, sun2022out}. While $\tau$-Far OOD detection is known to be uniformly learnable provided that the instance space $\mathcal{X}$ compact, below we take a closer look at its learnability.

As shown in Table 1, $\tau$-FSA is enough to guarantee non-uniform learnability -- even without the compactness assumption of~\cite{fang2022out}. Indeed, under $\tau$-FSA the closed ball $B(x, \tau)$ has no OOD risk for any sample $x \in \supp(D_{in})$. This allows for a learner constructed via a relaxation of the Maximal Zero OOD Risk Procedure by taking the union of the balls $B(x, \tau)$ for all training samples drawn.

\begin{restatable}[Far-OOD Detection]{theorem}{farood}
\label{farood}
OOD detection is non-uniformly learnable for $\mathcal{D}$ in $\mathcal{P}(\reals^{n})$ for a domain space $\mathcal{D}$ which satisfies $\tau$-FSA for a priori given $\tau > 0$,
\end{restatable}

A further question is whether $\tau$-Far OOD detection suffices to guarantee the uniform learnability of OOD detection. Unfortunately, the \nameref{freelunch} implies the contrary: the uniform learnability result in ~\cite{fang2022out} does not hold without the boundedness assumption. More precisely, if $\mathcal{D}^{\tau}$ is the domain space of all domains $D$ over $\reals^n$ satisfying $\tau$-FSA, we have:

\begin{restatable}{theorem}{nofsa}
\label{nofsa}
 The $\tau$-Far Support Assumption does \textbf{not} always imply the uniform learnability of OOD detection. In particular, uniform learnability of OOD detection is impossible for $\mathcal{D}^{\tau}$ in $\mathcal{P}(\reals^n)$ for any real constant $\tau >0$.
\end{restatable}

On the other hand, if a domain class $\mathcal{D}$ additionally satisfies BoundedID, the relaxation of the Maximal Zero OOD Risk Procedure discussed above does yield a uniform learner. This slightly relaxes the conditions in~\cite{fang2022out}, which require a compact instance space $\mathcal{X}$ for uniform learnability:

\begin{restatable}{theorem}{unfar}
\label{unfar}
Let $\mathcal{D}$ be a domain space over $\reals^n$ which satisfies the $\tau$-Far Support Assumption and the Bounded ID  Support Assumption for $R$ for a priori known positive real numbers $\tau$ and $R$. Then, OOD detection is uniformly learnable for $\mathcal{D}$ in $\mathcal{P}(\reals^n)$. 
\end{restatable}

\subsection{OOD Detection of Hölder Continuous Distributions}
\label{sec:holder}

Another area of great practical interest is OOD detection without the $\tau$-Far Supports Assumption, also known as Near-OOD detection \cite{fort2021exploring, ren2021simple, li2023rethinking}. That being said, there is a general lack of a theoretical understanding of how and under what conditions Near-OOD detection can be possible \cite{fang2022out}. Below, we aim to provide insight into this by considering two different scenarios which allow for Near-OOD detection.

In particular, we first look at Hölder continuity; the connection between distance-based classification and Holder continuity has been well established \cite{von2004distance}. Holder continuity is, however, a novel assumption in the context of OOD detection.  Perhaps somewhat surprisingly, we find that the learnability of OOD detection under the Hölder continuous ID/OOD Distribution Assumption behaves similarly to the learnability of Far-OOD detection. 

An intuitive explanation is that the combination of DSA and ID  Hölder or OOD Hölder forces the supports of $D_{out}$ and $D_{in}$ to be "far" from each other. Indeed, under OOD Hölder and DSA the amount of OOD mass close to a point in the support of $D_{in}$ is restricted.  Therefore, a low OOD risk hypothesis can be constructed by taking the union of balls with a sufficiently small radius centered at the training samples drawn, similarly to the learner of Far-OOD detection. 

Conversely, ID  Hölder and DSA force the mass of $D_{in}$ to be "far" from $\supp(D_{out})$. In particular, a learning algorithm can approximate the density of $D_{in}$ and return the regions in which the density of $D_{in}$ is sufficiently high. Thus, we present a distance-based algorithm in the OOD Hölder case and a density-based one in the ID  Hölder case:

\begin{restatable}[Non - uniform OOD Detection under Hölder]{theorem}{nonholder}
\label{nonholder}
    Let $\gamma > 0$ and $C \geq 0$ be two a priori given constants such that the domain space $\mathcal{D}$ satisfies DSA and the $(\gamma, C)$-Hölder Continuous ID  Distribution Assumption (or DSA and the $(\gamma, C)$-Hölder Continuous OOD Distribution Assumption). Then OOD detection is non-uniformly learnable for $\mathcal{D}$ in $\mathcal{P}(X)$.
\end{restatable} 

Denote with $\mathcal{D}_{H^{ID}_{\gamma, C}}$ the domain space of all domains over $\reals^n$ which satisfy DSA and the $(\gamma, C)$-Hölder Continuous ID  Distribution Assumption and with $\mathcal{D}_{H^{OOD}_{\gamma, C}}$ the domain space of all domains which satisfy DSA and the $(\gamma, C)$-Hölder Continuous OOD Distribution Assumption. Then:

\begin{restatable}[Uniform OOD Detection under Hölder]{theorem}{unholder}
\label{unholder}
    Let $\gamma > 0$, $C \geq 0$, and $R > 0$ be a priori given constants and let $\mathcal{D}$ be a domain space satisfying the $(\gamma, C)$-Hölder Continuous ID  Distribution Assumption (or the $(\gamma, C)$-Hölder Continuous OOD Distribution Assumption), DSA, and BoundedID for $R$. Then OOD detection is uniformly learnable for $\mathcal{D}$.

    The $(\gamma, C)$-Hölder Continuous ID  Distribution Assumption (or the $(\gamma, C)$-Hölder Continuous OOD Distribution Assumption), however, does \textbf{not} always imply the uniform learnability of OOD detection. In particular, uniform learnability of OOD detection is impossible for $\mathcal{D}_{H^{ID}_{\gamma, C}}$ and $\mathcal{D}_{H^{OOD}_{\gamma, C}}$ in $\mathcal{P}(\reals^n)$.
\end{restatable}

\subsection{OOD Detection under the ConvexID and DSA}
\label{sec:convno}
Another assumption that appears natural particularly in light of the Maximal Zero OOD Risk Procedure is the Convex ID Support Assumption. Crucially, ConvexID implies that for any domain $D$ and any set of training samples $S \in \supp(D_{in})^*$ the convex hull of $S$ lies entirely in $\supp(D_{in})$. Indeed, as $S$ is contained in the convex set $\supp(D_{in})$, the smallest convex subset of $\reals^n$ containing $S$, that is the the convex hull of $S$, will be a subset of $\supp(D_{in})$.
Thus, under ConvexID and DSA the Maximal Zero OOD Risk Procedure can be approximated by computing the convex hull of the drawn ID  samples. We call this the \textbf{Convex Hull Procedure}

We first remark that under DSA and ConvexID the Convex Hull Procedure trivially uniformly learns OOD detection over $\reals$ \cite{shalev2014understanding}. Unfortunately, for domains over higher dimensional instance spaces the Convex Hull Procedure does not always result in good approximations of the support of the ID  distribution. Indeed, consider an ID distribution $D_{in}$ which has $1 - \epsilon$ of its mass uniformly distributed on the boundary of a $2$-dimensional unit circle and the rest $\epsilon$ of its mass uniformly distributed in the inside of that circle. Then, the convex hull of any finite set of samples from $D_{in}$ will have cover at most $\epsilon$ of the mass of $D_{in}$. 

Even more so, an argument in the spirit of the proof of Theorem \ref{dsa:no} shows the following, where $\mathcal{D}_c$ is the domain space of all domains over $\reals^n$ which satisfy both DSA and ConvexID:

\begin{restatable}{theorem}{nocon}
\label{nocon}
     The combination of ConvexID and DSA does \textbf{not} always imply the non-uniform learnability of OOD detection. In particular, non-uniform learnability of OOD detection is \textbf{impossible} for the domain space $\mathcal{D}_c$ over $\reals^n$ in $\mathcal{P}(\reals ^n)$ for a positive integer constant $n \geq 2$.
\end{restatable}

\subsection{OOD Detection under DSA, ContID, and ConvexID}
\label{sec:convcont}

The impossibility result from the previous subsection relied on ID distributions whose boundaries contain most of their mass, Indeed, capturing parts of the boundary of $\supp(D_{in})$ in a hypothesis can be very challenging as it requires exactly learning the support of $D_{in}$ at least locally.

To alleviate this issue, we further require that the domain space $\mathcal{D}$ satisfies the Absolutely Continuous ID  Distribution Assumption (ContID). Indeed, as the boundary of a convex set
is always of Lebesque measure is $0$ \cite{lang1986}, ContID implies that for all $D \in \mathcal{D}$ the boundary of $\supp(D_{in})$ is of ID mass $0$.

\begin{restatable}{theorem}{conood}
\label{conood}
    OOD detection is non-uniformly learnable (via the Convex Hull Procedure) in $\mathcal{P}(\reals^n)$ for a domain space $\mathcal{D}$ over $\reals^n$ satisfying DSA, ConvexID, and ContID.
\end{restatable}

On the other hand, the Convex Hull Procedure does not uniformly learn OOD detection even under DSA, ConvexID, and ContID as for any fixed $N \in \nats$ we can construct a sufficiently precise approximation of the ID distribution supported on the unit circle with a heavy boundary from the previous section. Furthermore, similarly to how the infinite classical VC dimension of the convex sets leads to the hardness of supervised learning for convex concepts \cite{goyal2009learning, khot2008hardness, klivans2008learning}, here the following holds for the domain space $\mathcal{D}'_c$ of all domains over $\reals^n$ satisfying DSA, ContID, and ConvexID:

\begin{restatable}{theorem}{conun}
    \label{conun}
    The combination of DSA, ContID, and ConvexID does \textbf{not} always imply the uniform learnability of OOD detection. In particular, OOD detection is \textbf{not} uniformly learnable for $\mathcal{D}'_c$ over $\reals^n$ in $\mathcal{P}(\reals^n)$ for $n \geq 2$.
\end{restatable}

\section{Related Work}
\label{relwork}

\paragraph{OOD Detection Theory.} \cite{fang2022out} investigated the uniform learnability of OOD detection albeit in a slightly different setting. \cite{zhang2021understanding} discussed the impossibility of OOD detection under overlapping domains. Beyond this, \cite{morteza2022provable} introduced a class of provable learners of OOD detection for Gaussian mixtures and \cite{yang2021generalized} discussed the connection of OOD detection with problems such as Anomaly Detection, Novelty Detection and others.

\paragraph{OOD Detection in Practice.} In practical work, there are four broad classes of methods. Early \textbf{classification methods} used the maximum softmax probability as a proxy measure for ID  ness \cite{hendrycks2016baseline}. More recent classification methods have achieved improved performance via input perturbations and temperature-based scores \cite{liang2017enhancing}, energy-based scores \cite{liu2020energy, lin2021mood, NEURIPS2021_f3b7e5d3}, and by considering activation spaces \cite{sun2021react, dong2022neural}. \textbf{Density methods} model the density of the ID distribution and classify as OOD data from low density regions \cite{zong2018deep, morningstar2021density, choi2018waic, xiao2020likelihood, yang2023full}. However, training generative models can be computationally more expensive and performance is often worse than for classification methods. \textbf{Distance methods} are based on the assumption that OOD samples are geometrically separated from ID samples. Such methods utilize metrics such as the Mahalanobis distance \cite{lee2018simple, ren2021simple}, deep Nearest Neighbours \cite{sun2022out}, cosine similarity \cite{techapanurak2020hyperparameter, chen2020boundary, zaeemzadeh2021out}, kernel similarity \cite{van2020uncertainty}, distance metrics between points and centroids of classes \cite{huang2020feature, gomes2022igeood} and more. Lastly, \textbf{reconstruction-based methods} are based on the assumption that an the encoder-decoder framework would often yield different results for ID and OOD data. \cite{zhou2022rethinking, li2023rethinking, yang2022out}.

Additionally, a further class of algorithms relies on outlier exposure \cite{hendrycks2018deep,li2020background, yang2021semantically} which often allows for better performance. However, in this work we focus on the scenario in which no information about the OOD is provided during training.

\section{Discussion and Conclusions}

In conclusion, in line with supervised learning theory, our paper introduces the notion of uniform and non-uniform learnability of OOD detection, and shows that OOD detection is frequently non-uniformly learnable even though uniform learnability may not be possible. This may be a factor that explains the practical success of OOD detection even though it is impossible in theory. 

We also provide a principled algorithm for OOD detection with a provable guarantee of zero OOD risk, and show what it looks like under different assumptions. While some of the practical OOD detection algorithms look somewhat different from our algorithm, an avenue of future research is understanding how they relate under various assumptions on the ID data. 

In summary, our results have thus cleaned up and formalized the modeling assumptions surrounding the theory of OOD detection, and brought the theory a little closer to practice. How to bring it even closer under different definitions and modeling assumptions is another potential avenue for future work. 

\section{Discussion and Conclusions}

In conclusion, in line with supervised learning theory, our paper introduces the notion of uniform and non-uniform learnability of OOD detection, and shows that OOD detection is frequently non-uniformly learnable even though uniform learnability may not be possible. This may be a factor that explains the practical success of OOD detection even though it is impossible in theory. 

We also provide a principled algorithm for OOD detection with a provable guarantee of zero OOD risk, and show what it looks like under different assumptions. While most of the practical OOD detection algorithms look somewhat different from our algorithm, an avenue of future research is understanding how they relate under various assumptions on the ID data. 

In summary, our results have thus cleaned up and formalized the modeling assumptions surrounding the theory of OOD detection, and brought the theory a little closer to practice. How to bring it even closer under different definitions and modeling assumptions is another potential avenue for future work. 

\newpage

\bibliographystyle{unsrt}
\bibliography{paper.bib}

\newpage

\section*{Appendix A}

\subsection*{A.1 Detailed Related Work}

\paragraph{Near OOD Detection} As discussed in this work, OOD detection becomes more difficult in the absence of a separation parameter $\tau$ or a proxy for it, such as Hölder continuity. Indeed, Near-OOD detection comes with the additional challenges of similar and possibly even overlapping domains. In practice, \cite{ren2021simple} proposes relative Mahalanobis distance to measure the uncertainty of a classifier at a data point. Their method achieves good performance on a number benchmarks including: CIFAR-100 vs CIFAR-10 \cite{krizhevsky2009cifar}, CLINC OOD intent detection \cite{larson2019evaluation}, Genomics OOD \cite{ren2019likelihood}). \cite{fort2021exploring} utilizes large-scale pre-trained transformers and few-shot outlier exposure to further improve the AUROC on  CIFAR-100 vs CIFAR-10. Finally, \cite{li2023rethinking} finds that reconstruction-based methods such as Mmsked image modeling \cite{li2020background} could  significantly improve the performance of Near-OOD detection learners.

\paragraph{Anomaly Detection}

The goal of anomaly detection (AD) is to detect any samples which deviate from a defined during testing notion of normality \cite{chandola2009anomaly}. AD can be broadly classified into two classes: (i) sensory AD: AD under covariate datashift and (ii) semantic AD: AD under semantic datashift \cite{ruff2021unifying, ahmed2016survey}. The applications of sensory AD include: forgery recognition \cite{patel2016secure, nixon2008spoof, polatkan2009detection}, defense from adversarial attacks \cite{akhtar2018threat}, image forensics \cite{jiang2021deeperforensics, yang2020survey, dolhansky2019deepfake}, and others. On the other hand, the applications of semantic AD include crime surveillance \cite{idrees2018enhancing, diehl2002real}, image crawlers \cite{li2010optimol}, and others.

\paragraph{Novelty Detection}
The goal of novelty detection (ND) is the detection of test samples which do not fall into any training category \cite{pimentel2014review}. In general, there exist two types of Novelty detection based on the number of classification classes: (i) one-class novelty detection: when only one class label occurs in the training set (ii) and multi-class novelty detection: : when more than one class labels occurs in the training set. The scenario considered in this work is closely related to one-class novelty detection.  Applications of ND include: incremental learning \cite{al2015incremental, pathak2017curiosity}, video surveillance \cite{idrees2018enhancing, diehl2002real}, planetary exploration \cite{kerner2019novelty} and more.

\paragraph{Open Set Recognition} Open set recognition aims to detect test samples that do not fall into the categories observed at training time and accurately classify the test samples that fall into the observed categories \cite{liu2018open, dhamija2018reducing, scheirer2012toward}. Additionally, \cite{fang2021learning} introduces the problem of PAC-learnability of OOD detection.

\subsection*{A.2 Future Directions}

In this section, we outline several potential avenues of future work of interest. As already demonstrated in this and previous works, establishing a set-up for OOD detection which allows for good learnability results under mild and realistic conditions can be very challenging, despite the practical advances in OOD detection achieved in recent years. While, in this work we expand on previous theoretical work and provide positive results under some specific geometrical restrictions we still generally fall short of explaining this practical success for a diverse set of applications and benchmarks. Thus, we believe that further relaxations of the concept of learnability of OOD detection are needed. Here we propose the following potential directions:

As of yet, the theoretical performance of the learning is analyzed in terms of its worst-case risk. That is, for any learning algorithm we assume that the environment can adversarially pick the exact domain in the domain space to which the algorithm will be applied. On the other hand, when the performance of the various OOD detectors is evaluated in practice there is no such adversarial agent. We hypothesize that further positive results may be obtained if the learnability of OOD detection is considered in terms of average-case risks.

Another possible relaxation, is to consider a scenario in which at test time a whole batch of unlabeled data points are provided instead of just a single point. Alternatively, we can considered a scenario in which the detector is being updated online during testing utilized the test data. Note that both of these methods are used in practice to achieve better performance. In theory, it is easy to see how this could improve performance. For example, if the OOD probability is not a priori given we can use the batch to reliably estimate it. We believe that a more in-depth analysis of a similar set-up could yield more positive results for certain classes of domain spaces. 

A modification of the definition of risk is also possible. Firstly, we remark that it makes little sense to try and predict the pseudolabel $y$ of a data point $x$ which lies the overlap of $\supp(D_{in})$ and $\supp(D_{out})$, as $y$ is essentially a Bernoulli distributed random variable even is the domain $D$ is exactly known to us. Additionally, a (deterministic) classifier deployed subsequently will generally classify two identical points in the same way regardless of the distribution they were generated from. Thus, a more realistic approach might be to accept (i.e. predict that a point was generated from the ID distribution) with probability equal to the OOD probability $\alpha$ and that conditioned on a data point $x$ being accepted, the distribution of $x$ is close to $D_{in}$.

Finally, assuming that our ultimate goal is classification, the uncertainty in classifying an OOD data point might differ greatly based on its exact location in $\supp(D_{out})$. For example, while it might be hard to classify an OOD point very close to a class boundary or very far from the mass of $D_{in}$, it might also be relatively easy to generalize for an OOD point which lies in high density overlap $D_{in}$ and $D_{out}$ far from class boundaries. This suggests, that different more precise loss functions might  be of interest in practice and in theory, alleviating some of the difficulties with Near-OOD detection. This approach is closely related to classification methods for OOD detection.

\section*{Appendix B}
\label{sec:discdef}
\subsection*{B.1 Some Necessary Standard Definitions}

In this subsection, we introduce classical several notions we use throughout this work.

Let $(\mathcal{X}, T)$ be a topological space and let $\mathcal{B}(T)$ denote the Borel $\sigma$-algebra of $T$ - that is, the smallest $\sigma$-algebra of $T$ containing all open sets of $T$. Then the \textbf{support}
$\supp(\mu)$ of a measure $\mu$ on $(\mathcal{X}, \mathcal{B}(T))$ is the set of all elements $x \in \mathcal{X}$ for which every neighbourhood of $x$ has a positive $\mu$-mass. That is:

\begin{equation*}
    \supp(\mu) = \{x \in \mathcal{X} \mid \forall N_x \in T: x \in N_x \Rightarrow \mu(N_x) > 0\}
\end{equation*}

Following standard set theory notation we denote with $\mathcal{P}(\mathcal{X})$ the \textbf{power set} of $\mathcal{X}$ or the set of all subsets of $\mathcal{X}$.

\begin{definition}[Hölder Continuity]
    A real-valued function $f: \reals^n \rightarrow \reals$ is \textbf{$(\alpha, C)$-Hölder continuous} for constants $C \geq 0$ and $\alpha > 0$ if:
    \begin{equation*}
        |f(x)-f(y)|\leq C |x-y|^{\alpha}
    \end{equation*}
    holds for all $x$ and $y$ in $\reals^n$. Hölder continuity with $\alpha = 1 $ is usually referred to as \textbf{Lipschitz contunuity}.
\end{definition}

\subsection*{B.2 Some Initial Observations}
\label{sec:initobs}

We extend the discussion from the end Section \ref{sec:prelims} by formalizing it and presenting proofs for the claims made there.

\paragraph{Does a priori knowledge of the OOD-probability help?}

Generally, there are three modes of OOD-detection learnability based on the a priori given information about OOD-probability $\alpha$:

\begin{enumerate}[(i)]
    \item $\alpha \in (0,1)$ is fixed and a priori known
    \item $\alpha \in (0,1)$ is fixed but not a priori known
    \item $\alpha$ is not a priori fixed and can take any value in $[0,1]$
\end{enumerate}

While OOD-detection might appear to be easiest in Mode (i) and hardest in (iii), it turns out that in the realizable case, all three modes are equivalent. Indeed, we have the following straight-forward result:

\begin{restatable}[Equivalence of Learnability in Different Modes]{theorem}{thmodes}
\label{thmodes}
    Let OOD detection be realizable for the domain space $\mathcal{D}$ in the hypothesis space $\mathcal{H}$. Then OOD-detection is uniformly/non-uniformly learnable for $\mathcal{D}$ in $\mathcal{H}$ in Mode(i) if and only if  OOD-detection is uniformly/non-uniformly learnable for $\mathcal{D}$ in $\mathcal{H}$ in Mode(iii). 
\end{restatable}

\begin{proof}
As already metioned, learnability of OOD detection for a domain space $\mathcal{D}$ in a hypothesis space $\mathcal{H}$ in Mode (iii) trivially implies the learnability of OOD detection for $\mathcal{D}$ in $\mathcal{H}$ in Mode(i) as an algorithm can simply disregard its knowledge of the OOD probability.

Conversely, assume that the algorithm $\mathcal{A}$ learns OOD detection for a domains space $\mathcal{D}$ and a a priori given OOD probability $\alpha \in (0,1)$ in a hypothesis space $\mathcal{H}$. We need to construct an algorithm $\mathcal{A}'$ which for any $0 < \epsilon < \frac{1}{2}$ and $0 < \delta < \frac{1}{2}$ and any domain $D \in \mathcal{D}$ outputs a hypothesis $h'$ which with probability of at least $1 - \delta$ has risk $R^{\alpha'}_D(h') \leq \epsilon$ for all $\alpha' \in [0,1]$.

But as $R^{\alpha'}_D(h') = (1 - \alpha') R^{in}_D(h') + \alpha' R^{out}_D(h')$ it suffices to find a hypothesis $h'$ which with probability of at least $\delta$ satisfies:

\begin{equation*}
    R^{in}_D(h') \leq \epsilon \text{ and } R^{out}_D(h') \leq \epsilon
\end{equation*}

But this can be achieved by using $\mathcal{A}$ to find a hypothesis $h'$ which with probability of at least $1 - \delta$ satisfies:

\begin{equation*}
    R^{\alpha}_D(h') \leq \frac{\epsilon}{ \min\{\alpha, 1 - \alpha\}}
\end{equation*}

But this implies that $R^{in}_D(h') \leq \epsilon$ and $R^{out}_D(h') \leq \epsilon$ and, therefore, we have constructed a learner of OOD detection for $\mathcal{D}$ in $\mathcal{H}$ in Mode(iii). Clearly, if the sample complexity of $\mathcal{A}$ for any values $\epsilon$ and $\delta$ is uniformly bounded for all domains in $\mathcal{D}$ then the sample complexity of $\mathcal{A'}$ will also be uniformly bounded for $\mathcal{D}$. Thus, uniform learnability of OOD detection in Mode (i) also implies uniform learnability of OOD detection in Mode (iii).

\end{proof}

A corollary of the above and Theorem 1 in \cite{fang2022out} is that under a realizibility assumption the learnability of OOD detection (Mode (i)) does not depend on the exact value of the OOD probability $\alpha$. Hence, in most of our results we do not discuss the exact value of $\alpha$.

We remark that Theorem \ref{thmodes} requires the realizability assumption. Indeed, without it for different values of the OOD probability $\alpha$ risk close to the infimum might be achieved by disjoint sets of hypothesis. As the algorithm receives no information about the value of $\alpha$ during training, it is forced to essentially guess $\alpha$. Thus,  as discussed in \cite{fang2022out}, learning OOD detection in Mode (iii) requires that for every domain $D$ in a domain space $\mathcal{D}$ there is a subspace of the hypothesis space $\mathcal{H}$ which simultaneously optimizes the ID and the OOD risk. While this holds if $\mathcal{D}$ is realizable in $H$, it is a very strong assumption in an agnostic case practically making agnostic OOD detection in Mode (iii) impossible except under some unrealistic additional assumptions.

\paragraph{Overlapping Domains} An example of the phenomenon described above is the challanges in OOD detection for domains with an overlap between their ID and OOD distributions \cite{zhang2021understanding}. In particular, \cite{fang2022out} provided a definition for domains with overlapping distribution and showed the impossibility of OOD detection in Mode (iii) for domain spaces containing such domains in a sufficiently large hypothesis classes:

\begin{definition}[Overlap Between ID and OOD distributions \cite{fang2022out}] 
\label{def:overlap}
We say a domain $D$ has overlap between ID and OOD distributions, if there is a $\sigma$-finite measure $\mu$ such that $D_{in}$ and $D_{out}$ are absolutely continuous with respect to $\mu$, and
$ \mu(A_{\rm overlap})>0,~\textit{ where}$

$A_{\rm overlap}= \{\mathbf{x} \in \mathcal{X}:f_{in}(\mathbf{x})>0~ \textit{and}~ f_{out}(\mathbf{x})>0\}$. Here $f_{in}$ and $f_{out}$ are the representers of $D_{in}$ and $D_{out}$ in Radon–Nikodym Theorem \cite{cohn2013measure}, 
\begin{equation*}
  D_{X_{in}} = \int f_{in} {\rm d}\mu,~~~ D_{out} = \int f_{out} {\rm d}\mu.
\end{equation*} 
    
\end{definition}

As discussed above, the core of the argument is that due to the overlap the ID and OOD risks will be minimized for disjoint set of hypothesis and hence the output of a successful learner has to depend to the OOD probability about which the learner has no information during training time when in Mode (iii). In contrast, this argument does not imply for learnability in Mode (i) and OOD detection is indeed possible for some domain spaces containing domain with overlapping distributions.

Nonetheless, in this work we focus on realizable OOD detection for which we still have the following:

\begin{restatable}{proposition}{notreal}
\label{notreal}
 If a domain space $\mathcal{D}$ contains a domain with overlapping supports, OOD-detection for $\mathcal{D}$ is not realizable in any hypothesis space. 
\end{restatable}

\begin{proof}
Ihe proposition is equivalent to showing that OOD detection for $\mathcal{D}$ is not realizable in $\mathcal{P}(\mathcal{X})$ where $\mathcal{X}$ is the instance space. Therefore, it suffices to show that:

$$\inf_{h \in \mathcal{P}(\mathcal{X})} R^{\alpha}_D(h) > 0$$

But using the notation from Definition \ref{def:overlap}, we have that:

$$\inf_{h \in \mathcal{P}(\mathcal{X})} \geq \int_{A_{{\rm overlap}}} \min \{ (1 - \alpha)f_{in}, \alpha f_{out}\} {\rm d}\mu$$

But the inequality $\int_{A_{{\rm overlap}}} \min \{ (1 - \alpha)f_{in}, \alpha f_{out}\} {\rm d}\mu > 0$ can be shown using standard measure theory arguments. For more details see the proof of Theorem 3 from \cite{fang2022out}.
\end{proof}

\paragraph{Does a smaller hypothesis space help?} Finally, we remark that, similarly to supervised learning, under a realizibility assumption the learnability of OOD-detection depends only on the domain space and not on the hypothesis space. In particular, the following straightforward proposition holds:

\begin{proposition}
    \label{irrhyp}
    Let $\mathcal{D}$ be a domain space over $\mathcal{X}$ realizable in a hypothesis space $\mathcal{H}$. Then, if OOD detection is uniformly/non-uniformly learnable for $\mathcal{D}$ in $\mathcal{H}$, OOD detection is uniformly/non-uniformly learnable for $\mathcal{D}$ in $\mathcal{P}(\mathcal{X})$ 
\end{proposition}

\begin{proof}[Sketch of proof]. Given a learner $\mathcal{A}$ for OOD detection for $\mathcal{D}$ in $\mathcal{H}$ a new learner $\mathcal{A}'$ for OOD detection for $\mathcal{D}$ in $\mathcal{P}(\mathcal{X})$ can be constructed by first guessing the hypothesis space $\mathcal{H}$ and then applying $\mathcal{A}$.
    
\end{proof}

\subsection*{B.3 OOD-detection with non-trivial label spaces}
\label{sec:nontriviallabels}

In this subsection, we discuss the connection between OOD-detection with a non-trivial label space and a combination of OOD-detection with a trivial label space (one-class novelty detection) and standard supervised learnability. Indeed, a natural hypothesis is that detecting OOD for data with a non-trivial label space can be decomposed to first detecting ID-ness and then applying a standard supervised classifier. It turns out that only of this hypothesized equivalence holds and that is only under the Disjoint Supports Assumption.

Nevertheless, we need to modify our definition of OOD detection to allow for arbitrary label spaces. Denote with $\mathcal{Y}$ the set of labels observed during training and with $\mathcal{Y}'$ the set $\mathcal{Y} \cup \{ \perp \}$ where $\perp \notin \mathcal{Y}$ is a pseudolabel denoting that a data point is generated by $EX( D_{out})$. Note that in this case the two components $D_{in}$ and $D_{out}$ of a domain $D$ will be joint distributions over $\mathcal{X} \times \mathcal{Y}'$.

Additionally, for a loss function $l: (\mathcal{Y'}, \mathcal{Y'}) \rightarrow \reals_{\geq 0}$, the definition of the risk of a hypothesis $h$ becomes: 

\begin{equation*}
   R^{\alpha}_D(h) = \expect_{(x, y) \sim (1-\alpha)D_{in} + \alpha D_{out}}[l(h(x), y)]
\end{equation*}

For the sake of simplicity, we will again consider risk taken with respect to the 0-1 loss. However, it is not hard to adapt the result below for any loss function.

During training the samples are generated according by $EX(D_{in})$ and the algorithm observes their labels. In contrast to that, during testing the labels are not observed. Excluding these clarifications, the definitions of uniform and non-uniform learnability remain the same.  

Let $\mathcal{D}$ be a domain space over $\mathcal{X}$ with labels in $\mathcal{Y}'$. Denote with $\mathcal{D}^\mathcal{X} = (D_{in} \mid \mathcal{X} , D_{out} \mid \mathcal{X})$ the domain space formed by the marginal distributions on $\mathcal{X}$ of the domains in $\mathcal{D}$ and with $\mathcal{D}_{in}$ the set of all labelled ID-distributions. 

We say that a multi-label domain $D$ satisfies the Disjoint Supports Assumption if its marginal $D^\mathcal{X}$ satisfies it. That is if:

\begin{equation*}
    \supp(D_{in} \mid \mathcal{X}) \cap \supp(D_{out} \mid \mathcal{X}) = \emptyset
\end{equation*}

Finally, we can now state the result. Note that for the sake of brevity we only focus on the uniform learnability, however the result readily extends to non-uniform learnability.

\begin{theorem}
    Let $\mathcal{D}$ be a domain space over an instance space $\mathcal{X}$ with labels in $\mathcal{Y}'$ satisfying DSA. Then OOD detection is uniformly learnable for $\mathcal{D}$ in $\mathcal{Y}'^{\mathcal{X}}$ if OOD detection is uniformly learnable for $\mathcal{D}_x$ is $\mathcal{P}(\mathcal{X})$ and the domain space $\mathcal{D}_{in}$ is uniformly learnable in $\mathcal{Y}^{\mathcal{X}}$
\end{theorem}
   
\begin{proof}
    Assume that the algorithm $\mathcal{A}_1$ uniformly learns OOD detection for $\mathcal{D}_x$ in $\mathcal{P}(X)$ . Denote with $N_1: (0, 1/2)^2 \rightarrow \nats$ bounds the sample complexity of $A_1$. That is, for fixed $\epsilon$ and $\delta$ in $(0, 1/2)$ and domain $D^{\mathcal{X}} \in \mathcal{D}^{\mathcal{X}}$, $\mathcal{A}_1$ samples $N_1(\epsilon, \delta)$ and returns a hypothesis $h \in \mathcal{P}(\mathcal{X})$ which with probability of at least $\delta$ satisfies:

    \begin{equation*}
        R^{\alpha}_{D^{\mathcal{X}}}(h) \leq \epsilon
    \end{equation*}

    Note that as $\mathcal{D}$ satisfies DSA, we have $\inf_{h' \in \mathcal{P}(X)} R^{\alpha}_{D^{\mathcal{X}}}(h') = 0$ for all domains $D^{\mathcal{X}} \in \mathcal{D}^{\mathcal{X}}$
    
    Similarly, we assume that $\mathcal{A}_2$ is a uniform (supervised) learner for $\mathcal{D}_{in}$ in $\mathcal{Y}^{\mathcal{X}}$ with sample complexity bounded by $N_2 : (0, 1/2)^2 \rightarrow \nats$. We will now show how to use $\mathcal{A}_1$ and $\mathcal{A}_2$ to obtain a uniform learner for $\mathcal{D}$

    Let $\epsilon$ and $\delta$ be two constants in $(0,1/2)$ and let $D$ be any domain in $\mathcal{D}$. We run $\mathcal{A}_1$ with $N_1(\epsilon/2, \delta/2)$ samples and $\mathcal{A}_2$ with $N_2(\epsilon/2, \delta/2)$ samples to obtain two hypothesis $h_1: \mathcal{X} \rightarrow \{0,1\}$ and $h_2: \mathcal{X} \rightarrow \mathcal{Y}$, respectively.

    We now define a hypothesis $h: \mathcal{X} \in \mathcal\{Y\}'$ by setting:

    \begin{equation*}
        h(x) = 
        \begin{cases}
            h_2(x) \text{ if }   h_1(x) = 1 \\
            \perp \text{ if }  h_1(x) = 0
        \end{cases}
    \end{equation*}

    As the number of samples used to generate $h$ is bounded by $N_1(\epsilon/2, \delta/2)$ + $N_2(\epsilon/2, \delta/2)$ , to complete the proof it will suffice to show that for all $D \in \mathcal{D}$ with probability of at least $1 - \delta$ the following holds for all $h': \mathcal{X} \in \mathcal\{Y\}'$:

    \begin{equation*}
        R^{\alpha}_D(h) \leq R^{\alpha}_D(h') + \epsilon  
    \end{equation*}

    We can without loss of generality assume that $h'^{-1}(\perp) = \supp(D_{out})$. Indeed, we can redefine $h'(x)$ to be equal to $\perp$ on $\supp(D_{out})$ and to any arbitrary element $a$ of $\mathcal{Y}$ on $\supp(D_{out})^c$ can only decrease the risk of a hypothesis. As $D$ satisfies DSA this can only decrease the risk of $h$. 
    
    In addition, we define a new hypothesis $h'_2: \mathcal{X} \in \mathcal{Y}$ by setting:

    \begin{equation*}
        h'_2(x) = 
        \begin{cases}
            a \text{ if }  h'(x) = \perp \\
            h'(x) \text{ otherwise }
        \end{cases}
    \end{equation*}

    where again $a$ is an arbitrary fixed element of $\mathcal{Y}$. A direct union bound now implies that for all $h'$ with probability of at least $1 - \delta$ we simultaneously have that :

    \begin{equation*}
        R^{\alpha}_{D^\mathcal{X}}(h_1) \leq \epsilon \text{ and } R_{D_{in}}(h_2) \leq R^{\alpha}_{D_x}(h'_2) + \epsilon
    \end{equation*}

    But now it is easy to see that:

    \begin{align*}
        R^{\alpha}_{D}(h') &= (1 - \alpha) R_{D_{in}}(h'_2) \\
        &\geq (1 - \alpha) (R_{D_{in}}(h_2) - \epsilon) + \alpha (R^{\alpha}_{D^\mathcal{X}}(h_1) - \epsilon) \\
        &\geq (1 - \alpha) R_{D_{in}}(h) + \alpha R_{D_{out}}(h) + \alpha R^{\alpha}_{D^{\mathcal{X}}_{in}}(h_1) - \epsilon \\
        &\geq R^{\alpha}_{D}(h) - \epsilon 
    \end{align*}

    This concludes the proof.

\end{proof}

We end this section with a discussion about the limitations of the above result. First, note that the DSA assumption was crucial here. Indeed, without it the operation we used to obtain $h$ from $h_1$ and $h_2$ will not necessary produce a hypothesis which risk is close to $\inf_{h'} R^{\alpha}_{D}(h')$ even for $h_1$ and $h_2$ with small risk. Below we denote this operation with $\phi$. That is for hypotheses $h_1: \mathcal{X} \rightarrow \{0,1\}$ and $h_2: \mathcal{X} \rightarrow \mathcal{Y}$  we set $\phi(h_1, h_2)$ to be the result of the operation from the above proof.

For any $a \in \mathcal{Y}$ and any domain $D \in \mathcal{X}$ denote with $D_a$ the distribution on $\mathcal{X}$ of a the training samples labelled with $a$, that is: $D_a = (D_{in} \mid y = a)$

Assuming that $\mathcal{X} = \reals^n$ and that $D_{out}$ and all $D_a$ are absolutely continuous with respect to the Lebesque measure on $\reals^n$ with density functions $\mu_{out}$ and $\mu_{a}$, respectively, the hypothesis $h^*_2$ defined by:

\begin{equation*}
    h^*_2(x) = \argmax_{a \in \mathcal{Y}} \mu_a(x) \prob_{(x, y) \sim D_{in}} [y = a]
\end{equation*}

achieves the smallest possible risk for $D_{in}$ of any hypothesis in $\mathcal{Y}^{\reals^n}$. Similarly, the hypothesis $h^*_1$ defined by:

\begin{equation*}
        h^*_1(x) = \mathbbm{1}\{(1 - \alpha) \mu_{in}(x) \geq \alpha \mu_{out}(x)\}
\end{equation*}

achieves the smallest possible risk for the domain space $D^\mathcal{X}$ - the unlabelled variation of $\mathcal{D}$. Combining $h^*_1$ and $h^*_2$ we obtain a hypothesis $h = \phi(h^*_1, h^*_2)$ which at point $x$ first compares $\mu_{in}(x)$ and $\mu_{out}(x)$ and returns $\perp$ only if $(1 - \alpha) \mu_{in}(x) \geq \alpha \mu_{out}(x)$, however the hypothesis that achieves the optimal risk is:

\begin{equation*}
    h^*(x) = \argmax_{a \in \mathcal{Y}'} \mu_a(x) \prob_{(x, y) \sim D^{\alpha}} [y = a]
\end{equation*}

That is we return $\perp$ in all cases when $\alpha \mu_{out}(x)$ is higher than $\mu_a(x) \prob_{(x, y) \sim D_{in}} [y = a]$  for all $a \in \mathcal{Y}$. 

We can now easily construct and examples in which $h^*$ and $\phi(h^*_1, h^*_2)$ differ non-trivially. For example we can take $\mathcal{Y} = \{1 ,2\}$, the two labels being equally likely,  $D_{1} = D_{2} = D_{out} = U(0, 1)$, and $\alpha = 0.4$. Then $h^*$ achieves risk $0.6$ by  classifying with $\perp$ everything however $\phi(h^*_1, h^*_2)$ will classify all points with a label in $\{1,2\}$ and will have risk $0.7$.

Another natural question is whether the theorem holds in the opposite direction. In particular: \textit{Does the learnability of OOD detection for $\mathcal{D}$ imply the learnability of OOD detection for $\mathcal{D}^{\mathcal{X}}$ and the supervised learnability for $\mathcal{D}_{in}$ under some conditions?}. 

First note that, under DSA supervised learnability for $\mathcal{D}_{in}$ follow trivially from the learnability of OOD detection for $\mathcal{D}$. However, without DSA the implication does not hold as the hypotheses minimizing the risk for a domain $D$ contain no information about $D_{in}$ on $\supp{D_{out}}$.

On the other hand, we believe that under any reasonable assumptions OOD detection for $\mathcal{D}^{\mathcal{X}}$ does not follow from the learnability of OOD detection for $\mathcal{D}$. Indeed, $\mathcal{D}$ might be structured in such a way that the labels of the training samples reveal information about $\supp(D_{out})$ which would be inaccessible if we are working in $\mathcal{D}^{\mathcal{X}}$. Construction examples of this phenomenon under any of the assumptions introduced in this work is not hard.

\section*{Appendix C}
\label{sec:oodvc}

In this section, we will take a closer look into the connection between the learnability of OOD-detection and the VC dimension of the domains space. 

\subsection*{C.1 The No Free Lunch Theorem for OOD Detection}

To start with, we present a proof of the No Free Lunch Theorem stated in Section \ref{sec:nfl}.

\freelunch*

\begin{proof}

We will show that no uniform learner achieves expected risk smaller than $ \min\{\alpha, \frac{1}{2} (1 - \alpha)\}$ on all domains in $\mathcal{D}$. To this end, we assume the contrary. In particular, assume that the learning rule $\mathcal{A}$ achieves expected risk $r$ smaller than $\min\{\alpha, \frac{1 - \alpha}{2}\}$ for all domains in $\mathcal{D}$ with sample complexity of at most $N$ samples. 

Let $T$ be a subset of $\mathcal{X}$ of cardinality $3N$ shattered by $\mathcal{D}$. Thus, for every $A \subset T$ with $2N$ elements there exists a domain $D \in \mathcal{D}$  such that $\supp(D_{in}) \cap T = A$ and $\supp(D_{out}) \cap T = T/A$. As $\mathcal{D}$ is closed-under-mass-shifting the domains $D^A = (D^A_{in}, D^A_{out})$ with the same ID and OOD supports as $D$ and such that for all $s_1 \in A$ and $s_2 \in T \setminus A$ we have: 

$$\prob_{D^{A}_{in}}[s_1] = (1 - \epsilon)/2N \text{ and } \prob_{D^{A}_{out}}[s_2] = (1 - \epsilon)/N$$

where $\epsilon \leq 0$ is a small amount of "left-over" probability mass which we use to make the supports of $D^A_{in}$ and $D^A_{out}$ equal to  $\supp(D_{in})$ and $\supp(D_{out})$ respectively, if necessary.

In this way we can construct a subspace $D'$ of the a domain space $\mathcal{D}$ given by: 

$$\mathcal{D}' = \{D^A \mid A \subset T \text{, } |A| = 2N\}$$

We then pick a domain $D'$ uniformly at random from $\mathcal{D}'$ and consider the expected risk $L$ of the hypothesis returned by $\mathcal{A}$ when run on $D'$: 

$$L = \expect_{D' \sim \mathcal{D}'}[ \expect_{S \sim D'^{N}_{in}} [R_{D'}(\mathcal{A}(S))]]$$

As the learning rule $\mathcal{A}$ achieves expected risk less than $r$ for all $D \in \mathcal{D}$ we have that $L \leq r$. Additionally, denote with $B$ be the event that the drawn samples $S$ are all elements of $T$. For every $\delta > 0$ we can fix the value of  $\epsilon$ so that $\prob_{S \sim D'^{N}_{in}}[B] \geq 1 - \delta $ for all $D' \in \mathcal{D'}$. This means that the expected risk $L_B$ of the hypothesis $\mathcal{A}(S)$ conditioned on the event $B$ is at least $(1 - \delta) L$:

$$L_B = \expect_{D' \sim \mathcal{D}'}[ \expect_{S \sim D'^{N}_{in}} [R_{D'}(\mathcal{A}(S))\mid B] \geq (1 - \delta) L_B \geq (1 - \delta) r$$

Another way of generating the domain $D'$ and the set of training samples $S$ according to the same joint distribution is to first generate $S$ by uniformly sampling $N$ points from $T$ and then pick $D'$ uniformly at random from all the domains in $\mathcal{D}'$ consistent with $S$. These are all the domains $D \in \mathcal{D}'$ for which $S \in \supp({D_{in}})^N$. We denote the set of all domains consistent with $S$ with $D_S = \{D \in \mathcal{D} \mid S \in \supp({D_{in}})^N\}$. Thus, we obtain:

$$ L_B = \expect_{S \sim T^N}[ \expect_{D' \sim \mathcal{D}'_S} [R_{D'}(\mathcal{A}(S))] ]$$

We now bound $\expect_{D' \sim \mathcal{D}'_S} [R_{D'}(\mathcal{A}(S)]$ for a fixed $S \in T^N$. We have that:

\begin{align*}
\expect_{D' \sim \mathcal{D}'_S} [R_{D'}(\mathcal{A}(S))] &\geq \inf_{h \in \mathcal{P}(\mathcal{X})} \expect_{D' \sim \mathcal{D}'_S} [R_{D'}(h)] \\
&= \inf_{h \in \mathcal{P}(\mathcal{X})} \expect_{D' \sim \mathcal{D}'_S} [\sum_{s \in T/S} \mathbbm{1}_{\{h(s) \text{ is wrong}\}} \prob_{(1-\alpha)D'_{in} + \alpha D'_{out}}[s]] \\
&= \inf_{h \in \mathcal{P}(\mathcal{X})} \sum_{s \in T/S}  \mathbbm{1}_{\{h(s) = 0\}} \expect_{D' \sim \mathcal{D}'_S}[(1 - \alpha) \prob_{D'_{in}}[s]] + \mathbbm{1}_{\{h(s) = 1\}} \expect_{D' \sim \mathcal{D}'_S}[\alpha \prob_{D'_{out}}[s]]\\
&= \sum_{s \in T/S} \min\{ \frac{(1 - \alpha)(1 - \epsilon)}{2N} \prob_{D' \sim \mathcal{D}'_S}[ s \in \supp(D'_{in})] , \frac{\alpha(1 - \epsilon)}{N}] \prob_{D' \sim \mathcal{D}'_S}[ s \in \supp(D_{out}) \}\\
&= (3N - |S|) \min\{\frac{(1 - \alpha)(1 - \epsilon)}{2N} \frac{2N - |S|}{3N - |S|}, \frac{\alpha(1 - \epsilon)}{N} \frac{N}{3N - |S|}\}\\
&\geq \min \{ \frac{(1 - \alpha)(1 - \epsilon)}{2}, \alpha(1 - \epsilon)\}
\end{align*}

As the above holds for all $S$, we have that $L_B \geq \min \{ \frac{(1 - \alpha)(1 - \epsilon)}{2}, \alpha(1 - \epsilon)\}$, which implies: 

$$L \geq (1 - \delta)\min \{ \frac{(1 - \alpha)(1 - \epsilon)}{2}, \alpha(1 - \epsilon)\}$$

Thus, $r \geq \max_{D \in \mathcal{D}'} R_D(\mathcal{A})\geq L \geq (1 - \delta)\min \{ \frac{(1 - \alpha)(1 - \epsilon)}{2}, \alpha(1 - \epsilon)\}$. Letting both $\epsilon$ and $\delta$ tend to $0$ we obtain that $r \geq \min\{\alpha, \frac{1 - \alpha}{2}\}$ which is a contradiction with our assumptions

Therefore, the minimal expected risk a learning rule$\mathcal{A}$ can information-theoretically achieve for all domains in $\mathcal{D}$ is $\min \{\frac{1 - \alpha}{2}, \alpha \}$. In particular, OOD-detection is not uniformly learnable for $\mathcal{D}$ in $\mathcal{P}(\reals^n)$. Lastly, we remark that the argument presented above applies for both deterministic and randomized learning rules $\mathcal{A}$.
\end{proof}

As briefly mentioned, the closed-under-mass-shifting property on domain spaces required by the No Free Lunch Theorem for OOD Detection turns out to be too restrictive for all our purposes. In particular, we would like to be apply the No Free Lunch Theorem for domain spaces satisfying ContID (in Section \ref{sec:convcont}) and even ID-Hölder or OOD-Hölder (in Section \ref{sec:holder}). Clearly, such domains spaces are not closed-under-mass-shifting and the argument above is not possible as it requires the construction of domains whose mass in concentrated on a finite set of points contradicting the assumed absolute continuity of all domains. 

That being said, there is a natural adaptation of the No Free Lunch Theorem for such domain spaces. Indeed, instead of shattered set of points we can consider shattered families of disjoint non-null sets. Additionally, we note the proof presented above does not require the full freedom to move the probability mass of the ID and OOD distribution within their respective supports as required by the closed-under-mass-shifting property of the domain space but rather the ability the distribute a proportion of $1 - \epsilon$ of the ID and OOD probability mass on the elements (now non-null sets) of $T$. Summarizing, the analysis above we obtain the following generalized version of the No Free Lunch Theorem for OOD-Detection:

\begin{restatable}[Generalized No Free Lunch Theorem For OOD Detection]{theorem}{genfreelunch}
    \label{genfreelunch}
    Let the domain space $\mathcal{D}$ over instance space $\mathcal{X}$ be such that for all $N \in \nats$ and $\epsilon \in \reals_{>0}$ there exists a family $T$ of disjoint subsets of $\mathcal{X}$ such that:
    
    \begin{itemize}
        \item Every pair of a distribution of $(1 - \epsilon)$ of the ID probability mass on a subfamily $A \subseteq T$ and a distribution of $(1 - \epsilon)$ of the OOD probability mass on a subfamily of $T$ disjoint from $A$ is achieved by a domain in $\mathcal{D}$.
    \end{itemize}
    
    Then uniform OOD Detection for $\mathcal{D}$ in $\mathcal{P}(\mathcal{X})$ is impossible.
\end{restatable}

\begin{proof} 
    Analogous to the proof of Theorem \ref{freelunch}.
\end{proof}

\subsection*{C.2 Small VC Dimension Does Not Imply Learnability}

As we already discussed in the main body of the paper, the other direction of the Fundamental Theorem of Statistical Learning does not hold in the context of learnability of OOD Detection. Indeed, the domain subspace we used in the proof of Theorem \ref{dsa:no} to show the impossibility of non-uniform detection under DSA has VC dimension (as per Definition \ref{def:vc}) of $1$.

The complexity of the set of supports of the ID distributions, however, is not in any way reflected in the VC dimensionality of this domain space as its VC dimension is equal to 1 only because of all OOD distribution have point supports. Thus, a further question arises, namely: \textit{Does the finite or small VC dimensionalities of the ID supports and of the OOD supports imply the learnability of OOD Detection?}. We now discuss that even such more restrictive requirements do not imply learnability of OOD Detection.

First, observe that even for the same domain subspace used in the proof of Theorem \ref{dsa:no} the VC dimension of set of ID supports $S_{in} = \{\supp(D_{in}) \mid (D_{in}, D_{out}) \in \mathcal{D}\}$ is equal to $2$ while the VC dimension of the set of all OOD supports $S_{in} = \{\supp(D_{out}) \mid (D_{in}, D_{out}) \in \mathcal{D}\}$ is $1$, as mentioned above. Furthermore, the uniform learnability of OOD-detection is still not guaranteed even under the assumption that $\vc(S_{in}) = \vc(S_{out}) = 1$. We formalize this in the following proposition

\begin{proposition}
\label{novc1}
The uniform learnability of OOD Detection can be impossible even for domain spaces $\mathcal{D}$ which satisfy DSA and for which $\vc(S_{in}) = \vc(S_{out}) = 1$, that is the set of ID-supports $S_{in}$ and the set of OOD-supports $S_{out}$ are both of VC dimension $1$. 
\end{proposition}

\begin{proof}
We provide a counter example that is a domains space $\mathcal{D}$ which satisfies DSA and with $\vc(S_{in}) = \vc(S_{out}) = 1$, but for which the uniform learnability of OOD Detection is provably impossible.

For positive integers $m \leq n$, we consider a domain $D^{n,m} = (D^{n,m}_{in}, \{m\})$ with an ID distribution $D^{n,m}_{in}$ given by the probability mass functions $g_{n,m}: \nats \rightarrow [0,1]$ such that:

\begin{equation}
    g_{n,m}(x) = 
    \begin{cases}
        \frac{1}{n} $ if $  x \in \{1, \ldots ,n\} \setminus \{m\} \\
        \frac{1}{n} 2^{x - n} $ if $  x>n\\
        0 $ otherwise$
    \end{cases}
\end{equation}

We will consider the domain space $\mathcal{D}$ equal to $\{D^{n,m} \mid n, m \in \nats \text{ and } m \leq n\}$. The construction of $D^{n,m}$ above directly implies that DSA holds for all domains in $\mathcal{D}$. On the other hand, as all OOD distributions are point mass distributions and all ID distributions are supported on $\nats$ bar a single point $m$ we have that $\vc(S_{in}) = \vc(S_{out}) = 1$. Thus, it now remains to be shown that OOD-detection is not uniformly learnable for $\mathcal{D}$.

We assume that the learning rule $\mathcal{A}$ which with sample complexity of $N - 1$ for some positive integer $N$ achieves an expected risk at most $r$ for all domains in $\mathcal{D}$. Similarly to the argument used in the proof of the \nameref{freelunch}, we will first generate the $N$ samples and then randomly choose an exact distribution consistent with them.

In particular, we consider the expected risk $L$ of $\mathcal{A}$ when the domain is chosen uniformly at random from the domain subspace $\mathcal{D}^{2N} = \{ D_{2N, m} \mid m \in \{1, \ldots, 2N\}\}$. We then have:

$$L = \expect_{D \sim \mathcal{D}^{2N}}[ \expect_{S \sim D^N}[ [R_D(\mathcal{A}(S)] ] \leq r$$

On the other hand, another way of generating the domain $D$ and the set $S$ of $N - 1$ training samples according to the same joint distribution as above is to start by generating $N - 1$ samples according to a distribution $D'$ with a probability mass function $g'$ given by:

\begin{equation}
    g'(x) = 
    \begin{cases}
        \frac{2N-1}{(2N)^2}  \text{ if }  x \in \{1, \ldots ,2N\} \\
        \frac{1}{2N} 2^{x - 2N} \text{ if }  x>2N\\
        0 $ otherwise$
    \end{cases}
\end{equation}

and then choosing the exact distribution uniformly at random from the set $\mathcal{D}_{S}^{2N} = \{D \in \mathcal{D}^{2N} \mid S \in \supp(D_{in})^N\}$ of all domains in $\mathcal{D}^{2N}$ consistent with the already generated sample set $S$.

Therefore, similarly to the proof of the No Free Lunch Theorem we have:

\begin{align*}
    \expect_{D \sim \mathcal{D}_{S}^{2N}} [R_D(\mathcal{A}(S)] ] &\geq \inf_{h \in \mathcal{P}(\nats)} \expect_{D \sim \mathcal{D}_{S}^{2N}}[R_D(h)]\\
    &= \inf_{h \in \mathcal{P}(\nats)} \expect_{D \sim \mathcal{D}_{S}^{2N}}[ \sum_{s \in \{1, \ldots, 2N\} \setminus S} \mathbbm{1}_{\{h(s) \text{ is wrong}\}} \prob_{(1-\alpha)D_{in} + \alpha D_{out}}[s]] \\
    &= \inf_{h \in \mathcal{P}(\mathcal{X})} \sum_{s \in \{1, \ldots, 2N\} \setminus S}  \mathbbm{1}_{\{h(s) = 0\}} \expect_{D \sim \mathcal{D}_{S}^{2N}}[(1 - \alpha) \prob_{D_{in}}[s]] + \mathbbm{1}_{\{h(s) = 1\}} \expect_{D \sim \mathcal{D}_{S}^{2N}}[\alpha \prob_{D_{out}}[s]]\\
    &= \sum_{s \in \{1, \ldots, 2N\} \setminus S} \min\{ \frac{1 - \alpha}{2N} \prob_{D \sim \mathcal{D}_{S}^{2N}}[ s \in \supp(D'_{in})] , \alpha \prob_{D \sim \mathcal{D}_{S}^{2N}}[ s \in \supp(D_{out}) \} \\
    &= (2N - |S|) \min\{ \frac{1 - \alpha}{2N} \frac{2N - |S| - 1}{2N - |S|}, \alpha \frac{1}{2N - |S|}\} \\
    &\geq  \min\{ \frac{1 - \alpha}{2}, \alpha \}
\end{align*}

As this holds for all sets of samples $S$ of cardinality of at most $S$ we obtain that:

$$r \geq \max_{D \in \mathcal{D}'} R_D(\mathcal{A})\geq L \geq \min\{ \frac{1 - \alpha}{2}, \alpha \}$$

Thus we obtain that the minimal risk a learning rule $\mathcal{A}$ can information-theoretically achieve using $N-1$ samples for all domain spaces in $\mathcal{D}$ is $\min\{ \frac{1 - \alpha}{2}, \alpha \}$. Therefore, OOD Detection is not uniformly learnable for $\mathcal{D}$. This completes the proof.

\end{proof}

\subsection*{C.3 VC Dimension 1 Can Imply Learnability of OOD Detection}

We now discuss how the a VC dimension equal to $1$ of the set $S_{in}$ of ID supports in a domain space $\mathcal{D}$ can imply the uniform learnability of OOD Detection for $\mathcal{D}$ under some additional assumptions.

We call a map $TP: (\mathcal{X}, \{0,1\})^2 \rightarrow \{0,1\}$ a \textbf{two-point oracle} for a domain space $\mathcal{D}$ if $TP((x, i), (y, j)) = 1$ iff there exists a domain $D \in \mathcal{D}$ such that $\mathbbm{1}_{\{x \in supp(D_{in})\}} = i$ and $\mathbbm{1}_{\{y \in supp(D_{out})\}}= j$. 

Additionally, we will say that a domain space $\mathcal{D}$ has a \textbf{zero OOD-risk ID distribution} if there exists a domain $D  \in \mathcal{D}$ such that for any $D' = (D'_{in}, D'_{out}) \in \mathcal{D}$ we have that $P_{D'_{out}}(\supp(D_{in})) = 0$. 

\begin{theorem} \label{th:vc1}
Let $\mathcal{D}$ be a domain space satisfying DSA and for which the set of all ID-supports $S_{in}$ is of VC dimension 1. Additionally, assume that $\mathcal{D}$ has a zero OOD risk ID distribution and a two-point oracle $TP$ for $\mathcal{D}$ is given. Then OOD detection is uniformly learnable for $\mathcal{D}$ in $\mathcal{P}(\mathcal{X})$.  
\end{theorem}

In order to prove the above theorem we introduce the following terms:

\begin{definition}
    We say that a partial order $\preceq$ over a set $\mathcal{Y}$ is a tree ordering if for all elements $x$ of $\mathcal{Y}$ the initial segment $I_x = \{y \in Y \mid y \preceq x\}$ is linearly ordered under $\preceq$.
    .
\end{definition}

We can now state the following theorem, which we will use to show Theorem 4:

\begin{theorem}[\cite{ben20152}]
\label{partord}
Let $\mathcal{F}$ be a family of subsets of $\mathcal{X}$ of VC dimension 1 and let $f$ be an element of $\mathcal{F}$. Then, the relation $\preceq^{\mathcal{F}}_{f}$ defined by:

$$\preceq^{\mathcal{F}}_{f} = \{(x, y) \mid \exists h \in \mathcal{F}: h(y) \neq f(y) \rightarrow h(x) \neq f(x)\}$$

is a tree ordering in which for every $h \in \mathcal{F}$ the set $h_f = \{ x \in \mathcal{X} \mid h(x) \neq f(x) \}$ is an initial segment of $\preceq^{\mathcal{F}}_{f}$.

\end{theorem}

We are now ready to show Theorem 5.

\begin{proof}[Proof of Theorem \ref{th:vc1}]

Let the ID distribution $D^0_{in}$ of $D^0$ be the zero OOD risk ID distribution of $\mathcal{D}$ and let $\leq$ be the tree ordering on $\mathcal{X}$ defined by $\supp(D^0_{in})$ and $S_{in}$ as in Theorem \ref{partord}. Note that we can use the two-point oracle to compare two elements of $\mathcal{X}$ in $\leq$.

Let $\epsilon$ and $\delta$ be real numbers in $(0, 1/2)$ and let $D = (D_{in}, D_{out})$ be any domain in $\mathcal{D}$. Denote with $I$ the initial segment of $\leq$ those $x \in \mathcal{X}$ for which $X \in \supp(D_{in}) \oplus \supp(D^0_{in})$. Denote with $I_x \subseteq \mathcal{X}$ the set of all elements of $\mathcal{X}$ smaller or equal to $x$.

The algorithm is as follows: It first generates a set $S$ of $N(\epsilon, \delta)$ sample points from $D_{in}$. Let $s$ be the maximal sample point in $S$ in which $\supp(D_{in})$ and $\supp(D^{0}_{in})$. Note that as $S \cup I$ is a finite subset of a linearly ordered set it has a maximal element. The algorithm then returns the hypothesis $h = I_s \oplus \supp(D^{0}_{in})$. Note that using the two-point oracle $TP$ we can evaluate if a point $x \in h$ by computing whether $x \leq s$.

Outside of the initial segment $I$ the supports of $D_{in}$ and $D^0_{in}$ coincide, therefore on $\mathcal{X} \ I $ the hypothesis $h$ will only classify points in $\supp(D_{in})$ as in-distribution. Thus, as $\mathcal{D}$ satisfies DSA $h$ will have no risk outside on $I$. Similarly, $h$ coincides with $\supp(D_{in})$ on $I_s$ and, therefore, all its prediction will also be correct. To summarize, we the region $I \setminus I_s$ contains all points misclassified by $h$

In particular, it follows that the OOD risk $R^{out}_D(h)$ of $h$ equals $\prob_{D_{out}}[I /setminus I_s]$ and its ID risk $R^{in}_D(h)$ equals $\prob_{D_{in}}[I \setminus I_s]$.

However, as the domain space $\mathcal{D}$ satisfies DSA and  $D^0_{in}$ is a zero OOD-risk ID distribution for it we have that:

$$R^{out}_D(h) = \prob_{D_{out}}[I \setminus I_s] \leq \prob_{D_{out}}[I] \leq \prob_{D_{out}}[\supp(D^{0}_{in})] = 0$$

Therefore, we only need to consider the ID risk that is the quantity $\prob_{D_{in}}[I \setminus I_s]$.

But as in the proof of Theorem \ref{iofsa}, it is straightforward to verify that $ N(\epsilon, \delta) = \frac{1}{\epsilon} \ln (\frac{1}{\delta})$  samples guarantee that with probability of at least $\delta$ the ID mass of $I \setminus I_s$ and, therefore, the risk of $h$ is at most $\epsilon$, as desired. 
\end{proof}

\section*{Appendix D}
\label{sec:proofs}

In this section we present proofs for the results stated in Section \ref{ass_res}. 

\subsection*{D.1 OOD-detection under DSA}

We start with the following proposition:

\begin{restatable}{proposition}{lruledsa}
    \label{lrule}
     Let $\mathcal{A}$ be a (non-uniform) learning rule which given a domains $D$ has access to the oracle $EX(D_{in})$ and terminates with probability $1$ returning a hypothesis $h$. Then, there for exists a domain $D \in \mathcal{D}_S$ over $\reals^n$ such that:

    $$\expect[R_D(h)] \geq \alpha(1 - \alpha)$$

    where the expectations is taken over the random samples drawn and the internal randomization of $\mathcal{A}$.
\end{restatable}

\begin{proof}

Assume the contrary. Let $\mathcal{A}$ be a determisnistic learning rule that achieves an expected risk of at most $\lambda < \alpha(1 - \alpha)$ for all $D \in \mathcal{D}_s$ .
Denote with $I$ the interval $[0,1]$ and with $I^{\epsilon}_x$ the interval $(x, x + \epsilon)$ for any $\epsilon \in (0,1)$ and any $x \in [0, 1 - \epsilon]$. 

Consider a domain $D$ with $D_{in} = U(I)$, the uniform distribution over $I$, and $D_{out} = \{2\}$. Then the Monotone Convergence Theorem implies there exists $M \in \nats$ such that $\mathcal{A}$ terminates on $D_{in}$ with probability  $> 1 - \delta$ having drawn at most $M$ samples. We will denote this event with $B$. 


Note that we can without loss of generality assume that $\mathcal{A}$ always samples at least $M$ samples. This specifies a partial function $f: I^M\rightarrow \mathcal{P}(\reals^{n})$ defined whenever $B$ holds and given by $f(S) = \mathcal{A}(S)$ where $S$ is the sample set.

We also consider distributions:

\begin{equation*}
    D^{x}_{in} = (1 - \epsilon) U(I \setminus I^\epsilon_x) + \epsilon \cdot \{3\}
\end{equation*} 

for any $x \in [0, 1 - \epsilon]$, where we fix the value of $\epsilon \in [0,1]$ later. Note that the distributions $D^{x}_{in}$ and $D_{in}$ only differ on $I^\epsilon_x$ and on the single point set $\{3\}$.

Consider a function $g_x: I \rightarrow  I \setminus I^\epsilon_x \cup \{3\}$ given by:

\begin{equation*}
    g_x(X) =
    \begin{cases}
        3  \text{ if }   X \in I_x \\
        X  \text{ otherwise}
    \end{cases}
\end{equation*}

Now it is easy to see that if $X$ is distributed according to $D_{in}$, then $g_x(X)$ will be distributed according to $D^x_{in}$. Thus, we can simulate $EX(D^{x}_{in})$ by drawing a point from $EX(D^{in})$ and then applying $g_x$.

Denote with $h$ and $h_x$ the random values predictions of $\mathcal{A}$ when run on $D_{in}$ and $D^x_{in}$, respectively. By the above, we can use $g_x(EX(D^{in}))$ instead of $EX(D^{x}_{in})$ to generate samples from $D^x_{in}$. This does not affect the distribution of $h_x$.

Let $S_x = g_x(S)$ be the first $M$ samples generated by  $g_x(EX(D^{in}))$. Let $A_x$ be the event that the sample set $S$ does not intersect $I_x$. We have that:

$$\prob[A_x] = (1 -  \epsilon)^M  \geq 1 - M \epsilon$$. 

Clearly, if $A_x$ holds then $S = S_x$. Thus, if $f$ is defined on $S$ (that is if $B$ holds) the hypotheses $h$ and $h_x$ are both equal to $f(S)$. In particular, we have $h = h_x$.

Let $y$ be a point belonging to $I_x$. We set:

\begin{equation*}
    p_{x, y} = \prob_{S \sim (D_{in})^M}[ (y \in f(S)) \cap A_x \cap B]
\end{equation*}

Consider a $D_{x, y} = (D^{x}_{in}, \{y\})$. Then whenever $(y \in f(S)) \cap A_x \cap B$ occurs $\mathcal{A}$ achieves OOD-risk $1$ on $D_{x, y}$. Thus, the risk of $\mathcal{A}$ with respect to $D_{x, y} = (D_{in}, D_{y})$ is at least $\alpha p_y$. We obtain:

\begin{equation*}
\prob_{S \sim D^M_{in}}[ (y \in f(S)) \cap A_x \cap B] = p_y \leq \frac{\lambda}{\alpha} \text{ for all } y \in I_x
\end{equation*}

We now apply union bound to obtain the following for any $y \in [0, 1]^n$:

\begin{equation*}
\prob_{S \sim D^M_{in}}[ y \in h] \leq \prob_{S \sim D^M_{in}}[ (y \in f(S)) \cap A_x \cap B] + \prob[\neg A_x] + \prob[ \neg B] \leq \frac{\lambda}{\alpha} + M \epsilon + \delta
\end{equation*}

Letting $\epsilon$ and $\delta$ converge to $0$ we obtain:

\begin{equation*}
\prob[ y \in h] \leq \frac{\lambda}{\alpha}
\end{equation*}

We now show how this leads to a contradiction. To this end, consider the expected volume of the prediction $\text{vol}(h \cap I)$. We have that:

\begin{align*}
\expect_{D^{\infty}_{in}}[\text{vol}(h \cap I)] & = \int_{\mathcal{P}(R^{n})} \int_{I} \mathbbm{1}(x \in h) dx dh \\
& = \int_{I} \int_{\mathcal{P}(R^{n})} \mathbbm{1}(x \in h) dh dx \tag{ Fubini's Theorem } \\
& \leq \int_{I} \frac{\lambda}{\alpha} dx = \frac{\lambda}{\alpha}
\end{align*}

Finally, we now bound the ID-risk of $h$:

\begin{equation*}
\lambda = \expect_{h}[R_D(h)] \geq (1 - \lambda) \expect_{h}[R_{D_{in}}(h)] \geq (1 - \alpha) (1 - \frac{\lambda}{\alpha}) 
\end{equation*}

Rearranging, we obtain $\lambda \geq \alpha (1 - \alpha)$, as desired.

\end{proof}

\begin{remark}
    To extend the above argument to non-deterministic learning rules $\mathcal{A}$ we can simply consider the predictions $f(S)$ as random variables over $\mathcal{P}(\reals^n)$ instead of just elements of it.

    Note that, the contradiction was obtained using a very small subset of $D_s$. We can also modify the construction to only include absolutely continuous domains. 
\end{remark}

A corollary of the above is Theorem \ref{dsa:no}:

\dsano*

\begin{proof}
    Let $\mathcal{A}$ be an algorithm which non-uniformly learns OOD-detection for $D_{s}$ for an OOD-probability $\alpha$. 

    Fix the values of the learning parameters by setting $\epsilon = \delta = \alpha (1 - \alpha)/3$. Then, by assumption given access to $\epsilon$, $\delta$, $\alpha$, and $EX(D_{in})$ the algorithm $\mathcal{A}$ outputs a hypothesis $h$ which with probability of at least  $1-\delta$ achieves risk of at most $\epsilon$. But then:
    
    $$\expect[R_D(h)] \leq \delta + \epsilon < \alpha (1 - \alpha)$$
    
    But this contradicts Proposition \ref{lrule}
\end{proof}

\paragraph{A False Conjecture}

In this subsection we argue that the following conjecture stated in \cite{fang2022out} is actually false even for the more relaxed notion of uniform learnability in Mode (i):

\textbf{Conjecture}: \textit{If $\mathcal{H}$ is agnostic learnable for supervised learning, then OOD detection is learnable in $\mathcal{D}$ if for any domain $D \in \mathcal{D}$ and any $\epsilon>0$, there exists a hypothesis function $h_{\epsilon}\in \mathcal{H}$ such that:}

 \begin{equation*}
 h_{\epsilon}\in \{ h' \in \mathcal{H}: R_D^{\rm out}(h') \leq \inf_{h\in \mathcal{H}} R_D^{\rm out}(h)+\epsilon\} \cap   \{ h' \in \mathcal{H}: R_D^{\rm in}(h') \leq \inf_{h\in \mathcal{H}} R_D^{\rm in}(h)+\epsilon\}.
 \end{equation*}

 First, note that a if the domain space $\mathcal{D}$ is realized in $\mathcal{H}$ then the condition of the conjecture holds as $R^{\alpha}_D(h) \leq \epsilon$ implies that $R_D^{\rm in}(h') \leq \epsilon / ( 1 - \alpha)$ and $R_D^{\rm out}(h') \leq \epsilon / (1 - \alpha)$.  Additionally, by the Fundamental Theorem of Statistical Learning $\mathcal{H}$ is agnostically learnable if it has finite VC dimension.

  The proof of Proposition \ref{lrule} implies that uniform OOD detection is impossible for the domain space $\mathcal{D}$, defined below using the notation of the proof, in $\mathcal{P}(\reals)$:
  
  \begin{equation*}
      \mathcal{D} = (U(I), \{3\}) \cup \bigcup_{0 < \epsilon < 1}\{D_{x, y} \mid x \in [0, 1 - \epsilon], y \in (x, x + \epsilon)\}
  \end{equation*} 

  Thus, Proposition \ref{irrhyp} implies that OOD detection for $\mathcal{D}$ is impossible in any hypothesis space $\mathcal{H}$ which realizes $\mathcal{D}$. But an example of such hypothesis space is $\mathcal{H} = \{[0,1] \setminus x \mid x \in [0,1]\}$. But $\mathcal{H}$ has a VC dimension of exactly $1$ and, therefore, we have disproved the conjecture.

\subsection{D.2 OOD-detection under $\tau$-FSA}
\label{sec:fsaproof}

\nofsa*

\begin{proof}
    Follows directly from the \nameref{freelunch}: the set $\{2n\tau | n \in \{ 1, \ldots, m\}\}$ is shattered by $\mathcal{D}^{\tau}$ for all $m \in \nats$. Thus, the VC-dimension of $\mathcal{D}^{\tau}$ is infinite and Theorem \ref{freelunch} is applicable.
\end{proof}

\unfar*

\begin{proof}
Fix the values of $0 < \epsilon < 1/2$ and $0 < \delta < 1/2$ be two a and let $D$ be a domain belonging $\mathcal{D}$. We will use an algorithm based on the Maximal Zero OOD Risk Procedure.

More precisely, the algorithm generates a sample set of $N = N(\epsilon, \delta)$ points from $EX(D_{in})$ and returns a hypothesis $h$ which is evaluated to $1$ at a point $x \in R^n$ if $\min_{s \in S}\dist(x, s) \leq \tau$. In other words: $h = \cup_{s \in S} \cl(B(s, \tau))$.

We will now show that we can choose the sample complexity $N$ such that the risk of $h$ is at most $\epsilon$ with probability at least $1 - \delta$. First, note that as $\mathcal{D}$ is $\tau$-far we know that $h \subseteq \supp(D_{out})^c$. Therefore, we only need to consider $R_{in}(h)$.

Denote with $x_0$ be a point in $\supp{D_{in}}$ and let $B = \cl(B(x_0, R))$. As $\diam(\supp(D_{in})) \leq R$ we know that $\supp(D_{in}) \subseteq B$. Now partition $B$ into a set $Q$ of $M < (2\sqrt{n} R  /\tau)^n$ "small" open hypercubes of side $\tau/\sqrt{n}$. In addition, we add the boundaries (sides, edges, vertices) of the hypercubes in $Q$ to some of the hypercubes so that each point in  $\cl(B(x_0, R))$ belongs to exactly one of the hypercubes in $Q$. We now proceed with a standard PAC-learnability argument.

We call a hypercube in $Q$ significant if its $D_{in}$-mass is at least $\epsilon / M$. Note that if there exists a point $s \in S$ lying in hypercube $T \in Q$, then $T \subseteq h$ as $\diam(T) = \tau$. That means that it suffices to show that with high probability we sample set $S$ intersects all significant hypercubes as then the maximal risk $h$ can have is $\epsilon$.

To this end, for a significant hypercube $T$ in $Q$ let $A_T$ be the event that the sample set $S$ and $T$ are disjoint. Then, we have that:

$$\prob_{S \sim D^{N}_{in}}[A_T] = (1 - \prob_{D_{in}}[T])^N \leq \exp( - \epsilon N /M)$$

We can apply a union bound over all at most $M$ significant squares to obtain:

$$\prob_{S \sim D^{N}_{in}}[R_D(h) > \epsilon] <  M \exp( - \epsilon N /M)$$

This means that if $N = \frac{M}{\epsilon} \log (\frac{M}{\epsilon})$ the above algorithm is, indeed, a uniform learner with polynomial time and sample complexity.

\end{proof}

\farood*

\begin{proof}

Consider the procedure $\mathcal{B}$ from the proof of Theorem \ref{unfar}. Namely, for a set $S$ of sampled points the hypothesis $\mathcal{B}(S)$ is $h = \cup_{s \in S} \cl(B(s, \tau))$. As already discussed, the OOD-risk of $h$ is $0$.

On the other hand, the Monotone Convergence Theorem implies that for any $\epsilon > 0$ there exists a hypersphere $P$ with diameter $R_{\epsilon/2}$ which contains at least $1 - \epsilon / 2$ of the mass of $D_{in}$. But by the proof of Theorem \ref{unfar} the probability that $h$ contains at least $1 - \epsilon/3$ of the mass of $D_{in} \mid P$ tends to 1 as the number of samples goes to infinity. Combining the two inequalities we obtain:

\begin{equation*}
        \prob_{S \sim D^{n}_{in}}[R_{D_{in}}(\mathcal{B}(S)) < \epsilon] \rightarrow 1 \text{ as } n \rightarrow \infty
    \end{equation*}

This means that Proposition \ref{prtol} is applicable and, thus, OOD-detection is non-uniformly learnable for $\mathcal{D}$ as desired.

\end{proof}

\begin{remark}
A more hands-on algorithm (with no hypothesis testing) is also possible. Indeed, we can first approximate $R_{epsilon/2}$  using the samples $S$ after which we can directly run the algorithm from the proof of theorem for maximal risk $\epsilon/2$. See the proof of \ref{iofsa}.
\end{remark}

\subsection*{D.3 Generalized Far-OOD detection}
\label{sec:genfsa}

We will now prove Theorems \ref{nonholder} and \ref{unholder} by introducing a generalized version of Far-OOD Detection. We consider domains in which the ratio between the $D_{out}$ mass of an $\epsilon$-ball centered at any point $x \in \supp(D_{out})$ and the volume of this ball tends to $0$ as $\epsilon$ uniformly tends to $0$. More formally, we consider domains satisfying the following assumption, which we extend to domain spaces as per usual:

\begin{assumption}
    Let $g: R_{\geq 0} \rightarrow R_{\geq 0}$ be a function such that $\lim_{x \rightarrow 0} g(x) = 0$. 
    
    We say that a domain $D = (D_{in}, D_{out})$ over $\reals^n$ satisfies the \textbf{OOD $g$-Far Supports Assumption} if for all $\epsilon \in \reals_{> 0}$ and all $x \in \supp(D_{in})$:

    \begin{equation*}
        \prob_{D_{out}}[ \cl(B(x, \epsilon))] \leq g(\epsilon) \vol(B(x, \epsilon))
    \end{equation*}

    Similarly, we say that a domain $D = (D_{in}, D_{out})$ over $\reals^n$ is \textbf{ID $g$-Far Supports Assumption} if for all $\epsilon \in \reals_{> 0}$ and all $x \in \supp(D_{out})$:

    \begin{equation*}
        \prob_{D_{in}}[ \cl(B(x, \epsilon))] \leq g(\epsilon) \vol(B(x, \epsilon))
    \end{equation*}
\end{assumption}

Essentially, for a domain $D$ the OOD $g$-Far Supports Assumption requires that the average $D_{out}$ density (assuming that a density function exists) around a $\epsilon$-neighbourhood of a point $x \in \supp(D_in)$ tends to $0$ as $\epsilon$ tends to $0$. Hence, it is easy to see that the $(\gamma, C)$-Hölder Continuous OOD-Distribution Assumption and DSA imply the OOD $g$-Far Supports Assumption. Indeed, if a domain $D$ has a $(\gamma, C)$-Hölder Continuous OOD density function  and disjoint supports then the average OOD density of $\cl(B(x, \epsilon))$ is less than $C \epsilon^\gamma$. Similarly, $(\gamma, C)$-Hölder Continuous ID-Distribution Assumption and DSA imply the ID $g$-Far Supports Assumption.

Observe that the volume of an $n$-dimensional sphere with radius $r$: $B(x, r)$ where $x \in \reals^n$ is equal to $c_n r^n$, where $c_n$ is a constant which depends on the number of dimensions.

We now prove the positive learnability results of \ref{sec:holder} by showing them for any domain spaces satisfying the ID/OOD $g$-Far Supports Assumption. We begin with:

\begin{theorem}[Uniform OOD Detection under OOD $g$-FSA and BoundedID]
\label{oodfsa}
    Let the domain space $\mathcal{D}$ over $\reals^n$ satisfy the OOD $g$-FSA and BoundedID for some a priori given $g: \reals_{\geq 0} \rightarrow \reals_{\geq 0}$ and $R$. 
    Then OOD-detection is uniformly learnable for $\mathcal{D}$ in $\mathcal{P}(\reals^n)$. \cmmnt{The sample complexity and running time are polynomial in $1/\tau$, $1/\epsilon$, $1/\delta$, and $R$.}
\end{theorem}

\begin{proof}

Fix $0 < \epsilon < 1/2$ and $0 < \delta < 1/2$ and let $D \in \mathcal{D}$ be a domain in $\mathcal{D}$. As already mentioned above, we use the a similar algorithm as in Theorem \ref{oodfsa}.

More precisely, the algorithm generates a sample set $S$ of $N$ points from $EX(D_{in})$. Let $x_0$ be the first point sampled from in $\supp(D_{in})$ and let $B = \cl(B(x_0, R))$, which we partition in $M < (2\sqrt{n} R  /\tau)^n$ "small" $n$-dimensional cubes of side $\tau/\sqrt{n}$ for some real $\tau > 0$, which we will fix later.  Note that $\supp(D_{in}) \subseteq B$. Denote with $\mathcal{Q}$ the set of all the cubes in that partition. 

This time the algorithm returns a hypothesis $h$ consisting of all the cubes in $\mathcal{Q}$ from which we have generated samples in $S$. That is, the algorithm returns:

\begin{equation*}
    h = \bigcup_{Q \in \mathcal{Q} \wedge S \cap Q \neq \emptyset} Q
\end{equation*}

Note that as in this case we do not necessarily have that $h$ is disjoint from $\supp(D_{out})$, $h$ can have positive OOD and ID risk. However, the the analysis of the ID-risk from Theorem \ref{unfar} still applies here and allows us to bound the ID risk. Indeed, for a sample set of $N$ points. where  

\begin{equation*}
    N = \frac{M}{\epsilon} \log (\frac{M}{\epsilon})
\end{equation*}

we know that with probability of at least $1 - \delta$:

\begin{equation*}
    R_{in}(h) \leq \epsilon
\end{equation*}

On the other hand, as $D$ is $g$-separated and the diameter of each hypercube $Q$ in $\mathcal{Q}$ is $\tau$ we know that if $Q \cap \supp(D_{in}) \neq \emptyset$ then $\prob_{D_{out}}[Q] \leq g(\tau) c_n \tau^n$. This implies that:

\begin{equation*}
    R_{out}[h] \leq c_n g(\tau) \tau^n M \leq c_n (2\sqrt{n} R)^n g(\tau)
\end{equation*}

In particular, we can choose $\tau = g^{-1}(\epsilon c_n (2\sqrt{n} R)^n) $ which will guarantee
$R_{out}[h] \leq \epsilon$. This completes the proof.

\end{proof}

Similarly, we have that:

\begin{theorem}[Uniform OOD Detection under ID $g$-FSA and BoundedID]
\label{idfsa}
    Let the domain space $\mathcal{D}$ over $\reals^n$ satisfy the ID $g$-FSA and BoundedID for some priori given $g: \reals_{\geq 0} \rightarrow \reals_{\geq 0}$ and $R$. 
    Then OOD-detection is uniformly learnable for $\mathcal{D}$ in $\mathcal{P}(\reals^n)$.\cmmnt{The sample complexity and running time are polynomial in $1/\tau$, $1/\epsilon$, $1/\delta$, and $R$.}
\end{theorem}

\begin{proof}
    Again, let $0 < \epsilon < 1/2$ and $0 < \delta < 1/2$ and let $D  \in \mathcal{D}$ be any domain in $\mathcal{D}$. As per usual, let  $x_0$ be the first sample point and let $B = \cl(B(x_0, R))$ which we cover with the set $\mathcal{Q}$ of $M < (2 \sqrt{n} R /\tau)^n$ $n$-dimensional cubes of side $\tau/\sqrt(n)$. 

    The algorithm generates $N$ sample points and then return a hypothesis containing all hypercubes of $\mathcal{Q}$ containing at least $l$ sample points, where $l \in \reals$ and $N \in \nats$ are parameters which we will fix later. That is:

    \begin{equation*}
        h = \bigcup_{Q \in \mathcal{Q} \wedge |S \cap Q| \geq l} Q
    \end{equation*}

    Before we begin with the analysis of the algorithm, we observe that if $Q \cap \supp(D_{out}) \neq \emptyset$ then $\diam(Q) = \tau$ implies that $\prob_{D_{out}} (Q) \leq c_n g(\tau) \tau^n $

    Now as before, we call a hypercube $Q \in \mathcal{Q}$ significant if $P_{D_{in}}[Q] \geq A = \epsilon/M$ and insignificant if $P_{D_{in}}[Q] \leq B = c_n g(\tau) \tau^n$. Note that if all significant hypercubes are included in $h$ while all insignificant are not in $h$ then $h$ will have no OOD risk, while the ID risk of $h$ will be at most $\epsilon$ . 
    
    We will now show that we can choose the values of $\tau$, $N$, and $l$ so that the above holds with probability of at least $1 - \delta$. To this end, we fix the value of $\tau$ so that  $\epsilon / 4 c_n (2 \sqrt{n} R)^n > g(\tau)$. Note that this implies that $A > 4B$. We also fix $l$ to be equal to $2B = 2c_n g(\tau) \tau^n$.

    Let $Q \in \mathcal{Q}$ be an insignificant square. Then a direct application of Chernoff bounds implies that:
    
    \begin{equation*}
        \prob_{S \sim D_{in}^{N}} [Q \subseteq h] \leq \exp(-NB/3)
    \end{equation*}

    Similarly, if $Q \in \mathcal{Q}$ is a significant square, Chernoff bounds imply that: 

    \begin{equation*}
        \prob_{S \sim D_{in}^{N}} [Q \not\subseteq h] \leq \exp(-NB/2) <  \exp(-NB/3)
    \end{equation*}

    A union bound now implies that the algorithm is not successful with probability of at most:

    \begin{equation*}
        P_{S \sim D_{in}^{N}}[R(h) > \epsilon] > M \exp(-NB/3)
    \end{equation*}

    This implies that it suffices to choose $N \geq \frac{3}{B}\log(\frac{M}{\delta})$. Substituting everything in, we obtain that the sample complexity $N$ of the algorithm is:

    \begin{equation*}
        N  \leq \frac{3}{c_n g(\tau) \tau^n}\log\bigg(\frac{(2 \sqrt{n} R)^n}{\delta \tau^n}\bigg)
    \end{equation*}

    where $\tau = g^{-1}(\epsilon / 4 c_n (2 \sqrt{n} R)^n) $

\end{proof}

We now modify the algorithms from the two theorems above to obtain non-uniform learners of OOD Detection under ID or OOD $g$-FSA without the boundedness assumptions. We do that by first finding a sphere that with high probability  contains most of the mass of the ID distribution and then applying Theorem \ref{oodfsa} or Theorem \ref{idfsa}.  

We remark that under OOD $g$-FSA we can use Proposition \ref{prtol} similarly to how Theorem \ref{farood} followed from Theorem \ref{unfar}. This is possible because the hypotheses returned be the algorithm in Theorem \ref{oodfsa} have OOD risk of at most $\epsilon$ with probability $1$. This, however, is not the case for ID $g$-FSA and, and for that reason, we present a more general argument.

\begin{theorem}[Non-uniform OOD Detection under ID or OOD $g$-FSA]
\label{iofsa}
OOD-detection is non-uniformly learnable for $\mathcal{D}$ in $\mathcal{P}(\reals^{n})$ for a domain space $\mathcal{D}$ which satisfies ID $g$-FSA or OOD $g$-FSA for some a priori given function $g: R_{\geq 0} \rightarrow R_{\geq 0}$.
\end{theorem}

\begin{proof}
Again, fix the valuse of  $0 < \epsilon < 1/2$ and $0 < \delta < 1/2$ and let $D =\in \mathcal{D}$ be any domain in the domain space $\mathcal{D}$

The algorithm begins drawing samples: denote with $x_0$ the first sample. Let $y$ be the point furthest away from $x_1$ after the first $N_1$ samples. Observe that we can fix the value of $N_1$ so that a closed ball $B = \cl(B(x_0, ||x -y||))$ contains at least $1 - \epsilon/2$ of the mass of $D_{in}$ with probability of at least $1 - \delta/2$. Indeed, it is straightforward to verify that $N_1 = \frac{2}{\epsilon} \ln(\frac{2}{\delta})$ suffices.

We consider the domain $D'$ where $D'_{in} = D_{in} \cap B$ is the part of $D_{in}$ contained in the closed ball $B$ and $D'_{out} = D_{out}$. We can now apply Theorem \ref{oodfsa} or Theorem \ref{idfsa} to obtain a hypothesis $h$ which $D'$ risk is at most $\epsilon/2$ with probability of at least $1 - \delta/2$ by simply ignoring the samples that lie outside $B$. 

Finally, a simple union bound implies that with probability of at least $1 - \delta$ we have that:

\begin{equation*}
    \prob_{D_{in}}[B] \geq 1 - \epsilon/2 \text{ and } R_{D'}(h) \leq \epsilon/2
\end{equation*}

But then $R_D(h) \leq \epsilon$. This completes the proof.

\end{proof}

\paragraph{OOD Detection under Hölder}

Finally, we use the results above and the No Free Lunch Theorem to prove Theorem \ref{nonholder} and \ref{unholder}:

\nonholder*

\begin{proof}
Let $D$ be a domain satisfying the $(\gamma, C)$-Hölder Continuous ID Distribution Assumption and let $f_{in}$ be the Hölder continuous density function of $D_{in}$. This means that for any  $x \in \supp{D_{out}}$ and $y \in \reals^n$ such that $||x - y|| \leq \tau$ we have:

\begin{equation*}
f_{in}(y) \leq C \tau^\gamma
\end{equation*}

Integrating on $\cl B(x, \tau)$ we obtain that:

\begin{equation*}
\frac{\prob_{D_{in}}[B(x, \tau)]}{\vol(B(x, \tau))} = C \tau^\gamma / c_n \rightarrow 0
\end{equation*}

This means that $\mathcal{D}$ satisfies the ID $g$-Far Supports Assumption. Therefore, the conditions of Theorem \ref{iofsa} hold and OOD Detection for $\mathcal{D}$ is non-uniformly learnable.

Similarly, we can verify that the $(\gamma, C)$-Hölder Continuous OOD Distribution implies the OOD $g$-Far Supports Assumption and, thus, Theorem \ref{iofsa} implies the non-uniform learnablity of OOD Detection for $\mathcal{D}$
\end{proof}

Finally, we complete this subsection by presenting a proof of Theorem \ref{unholder}

\unholder*

\begin{proof}
In the proof of Theorem \ref{unholder} we saw that the domain space $\mathcal{D}$ satisfies the ID or the OOD $g$-Far Supports Assumption. As $\mathcal{D}$ also satisfies BoundedID, Theorems \ref{idfsa} and \ref{oodfsa} imply that OOD detection for $\mathcal{D}$ is indeed uniformly learnable.

To complete the proof we will now show that uniform OOD detection is not always possible under only ID-Hölder or OOD-HölderOOD. For the sake of clarity, we will consider domains over $\reals$ with $(1, 1)$-Hölder continuous ID and OOD density functions, however the argument readily extends for any Hölder continuous domain spaces over $\reals^n$.

As the distributions are absolutely continuous we will apply the Generalized No Free Lunch Theorem. To this end, we consider the family $\mathcal{F} = \{[10i, 10i+ 3] \mid i \in \nats\}$ . It is easy to see that $\mathcal{F}$ satisfies the condition of the \nameref{genfreelunch} for the domain space $\mathcal{D} = \mathcal{D}_{H_{1, 1}^{ID}} \cap \mathcal{D}_{H_{1, 1}^{OOD}}$. In particular, the domains in $\mathcal{D}$ achieve any distribution of their ID and OOD mass on $\mathcal{F}$. Indeed, for any such distribution on $\mathcal{F}$ a Hölder continuous domain can be constructed by first increasing and then decreasing the density of $D_{in}$ or $D_{out}$ with gradient $1$ inside the supported intervals $[10n, 10n+3]$ achieving the desired probability mass. 

Thus, OOD Detection is not uniformly learnable for $\mathcal{D}_{H_{1, 1}^{ID}}$ and $\mathcal{D}_{H_{1, 1}^{OOD}}$ in $\mathcal{P}(\reals^n)$.
\end{proof}

\subsection*{D.4 OOD Detection under ConvexID }

Below we explain explain how we can modify the argument from Theorem \ref{dsa:no} and extend it to domains spaces which satisfy ConvexID to prove Theorem \ref{nocon}

\nocon*

\begin{proof} Assume that OOD detection is non-uniformly learnable for the domain space $\mathcal{D}_c$. We will only construct the subfamily of $\mathcal{D}_c$ with which we can obtain a contradiction with this as afterwards the argument closely follows the proof of Proposition \ref{lrule}.

In particular, as "base" ID distribution $D_{in}$ we will consider a distributions supported on the two-dimensional unit circle $B^2$ with $\lambda$ of the its mass distributed uniformly on the its inside $\inter(B^2)$ and the rest $1 - \lambda$ of the mass distributed uniformly on its boundary $\partial B^2$. $D_{in}$ will correspond to the distribution $U([0,1])$ in of Proposition \ref{lrule}. 

For the ease of notation will refer to points by by their spherical coordinates $(\theta, r)$ for $\theta \in [0, 2\pi)$ and $r \in \reals_{\geq 0}$. We fix a small $\epsilon < \pi$ and for a $\theta \in [0, 2\pi)$  we denote with $C_{\theta}$ the larger of the two regions to which the line $(\theta, 1)$ and $(\theta + \epsilon, 1)$ splits $B^2$. Additionally, denote with $\kappa_{\epsilon}$ the $D_{in}$ mass of the smaller of these regions, which does not depend on the exact choice of $\theta$. 

The distribution $D^{\theta}_{in}$ close to $D_{in}$ we will consider are given by:

\begin{equation*}
    D^{\theta}_{in} = (1 - \kappa_{\epsilon}) D_{in} \mid C_{\theta} + \kappa_{\epsilon} \{(0, 2)\}
\end{equation*}

They correspond to $D^{x}_{in} = = (1 - \epsilon) U([0,1] \setminus [x, x + \epsilon]) + \epsilon \cdot \{3\}$ in the proof of Proposition \ref{lrule}. Having defined the ID distribtutions, we can now follow the argument from the proof of Proposition \ref{lrule}. In particular, letting $\epsilon$ and $\lambda$ tend to $0$ we obtain that no algorithm achieves error less than $\alpha(1 - \alpha)$ on all domains in $\mathcal{D}_c$.

\begin{center}
    \includegraphics[width=80mm,scale=1.5]{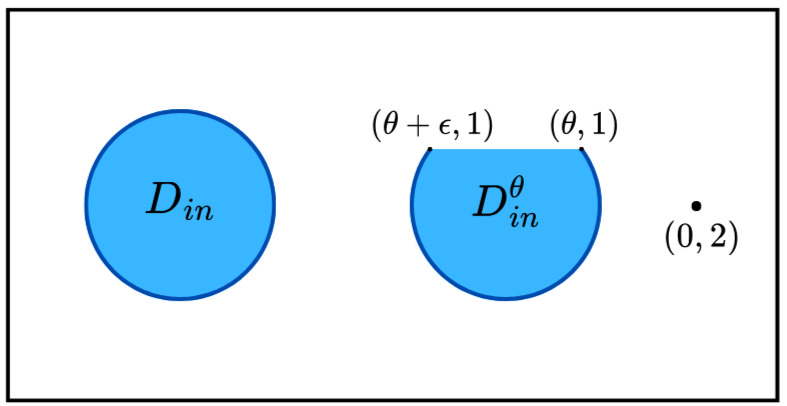}\\
    Figure 2. 
\end{center}

\end{proof}

Next, we present a proof of Theorem \ref{conood}:

\conood*

\begin{proof}
    As mentioned in the main body of this work, we show that Convex Hull Procedure satisfies the conditions of Proposition \ref{prtol}. We denote with $h_M$ be the hypothesis produced by the Convex Hull Procedure after the $M$ data point is sampled form $D_{in}$. That is, $h_M$ is the convex hull of the first $M$ samples. As already remarked, as $\supp(D_{in})$ is a convex set all the hypotheses $h_M$ will have no OOD risk. Therefore, it remains to show that for all domains $D \in \mathcal{D}$:

    \begin{equation*}
        \prob_{S^M \sim D^M_{in}}[R_{D_{in}}(h_M) < \epsilon] \rightarrow 1 \text{ as } n \rightarrow \infty
    \end{equation*}

    Note that a point $x \in \inter \supp(D_{in})$. $x$ belongs to $h_M$ iff every half-space defined with a hyperplane passing through $x$ contains a sample point in $S_M$. 
    
    Denote with $L_{\lambda}$ the set of all points $x \in \inter K$ which Tukey depth \cite{tukey1975mathematics} is at least $\lambda$. We will fix the value of $M$ so that with probability of at least $1 - \delta$, $h_M$ contains the entire $L_{\lambda}$. 

    The VC dimension of the set of all half-spaces in $\reals^n$ is well-known to be $n+1$ \cite{shalev2014understanding}. A standard generalization bound (as in Lemma 7 of \cite{chaudhuri2010rates}) yields that with probability of at least $ 1 - \delta$ every half-space $H$ with:
    
    \begin{equation*}
        \prob_{D_{in}}[H] > \frac{C_{\delta}d\log(M)}{M}
    \end{equation*}
    
    intersects $h_M$ where $C_{\delta}$ is constant that only depends on $\delta$ and $C_{\delta} \in O(\ln(1/\delta))$. That implies that for $M > \frac{C_{\delta}d}{\lambda} log^2(C_{\delta}d/\lambda)$ with probability of at least $1-\delta$, $h_M$ contains the entire $L_{\lambda}$.
    
    Thus, to complete the proof it suffices to show that for all domains $D \in \mathcal{D}$ we have:
    
    \begin{equation*}
        \lim_{\epsilon \rightarrow 0} \prob_{D_{in}}(L_{\epsilon}) = 1
    \end{equation*}

    But this follows readily as all $x \in \inter (\supp(D_{in}))$ have positive Tukey depth. A standard application of MCT then implies that:
    
    $$\lim_{\lambda \rightarrow 0} \prob_{D_{in}}[L_{\lambda}] = \prob_{D_{in}}[\inter (\supp(D_{in}))] = 1 - \prob_{D_{in}}[\partial (\supp(D_{in}))] = 1$$
    
    as $D_{\in}$ is absolutely continuous.

\end{proof}

\begin{remark}
    The sample and time complexity of the convex hull algorithm for a domain space depends on how quickly the $D_{in}$ mass of the sets $L(\epsilon)$ tends to $1$. In particular, if this convergence is polynomial for all domains $D \in \mathcal{D}$ then the convex hull will require sample and time complexity polynomial in $1/\epsilon$ and $1/\delta$.

    Intuitively, we can think of the behaviour of $1 - \prob_{D_{in}}[L(\epsilon)]$ as a proxy for the complexity of the ID-distribution of a domain space. It captures how heavy the boundary of $\supp{D_{in}}$ and its neighbourhood are.
\end{remark}

Lastly, we present a proof of the impossibility result of Theorem \ref{conun}:

\conun*

\begin{proof}
As already mentioned, we will apply the \nameref{genfreelunch}. For the sake of clarity, we work in $\reals^2$.

To this end, for any $n \in \nats$ we construct a regular $n$-gon $A_1 A_2 \ldots A_n$ with side 1 . We construct the family of sets $T = \{B(A_i, \epsilon_n) \mid i \in \{1, \ldots, n\}\}$ for $\epsilon_n > 0$ sufficiently small to guarantee that the family $T$ is shattered by $\mathcal{D}'_c$. 

It is easy to see that for any $\epsilon > 0$ it is possible to distribute $1 - \epsilon$ of the ID mass on any subfamily of $T$ (and use the remaining $\epsilon$ of the ID mass to complete a convex ID support). Thus, the conditions of the \nameref{genfreelunch} hold and, therefore, we obtain that the uniform learnability of OOD detection is impossible for the domain space $\mathcal{D}'_c$.

\end{proof}

\end{document}